\documentclass[accepted]{uai2021}

\usepackage{microtype}
\usepackage{graphicx}
\usepackage[TABBOTCAP]{subfigure}
\usepackage{booktabs} 
\usepackage{xcolor}
\usepackage{bbm}
\usepackage{mathtools}
\usepackage{natbib}
\usepackage{verbatim}
\usepackage[linesnumbered, ruled,vlined]{algorithm2e}
\usepackage{natbib} 
\usepackage{algorithm,algorithmic}
\usepackage{blindtext}
\usepackage{abstract}

\usepackage{enumitem}

\usepackage{microtype}
\usepackage{graphicx}
\usepackage{subfigure}
\usepackage{booktabs} 
\usepackage{wrapfig}
\usepackage{graphicx}

\usepackage{fdsymbol}

\usepackage{enumitem,kantlipsum}

\usepackage{url}
\usepackage{graphicx}
\usepackage[font=scriptsize]{caption}
\usepackage{amsmath}

\usepackage{amsfonts}
\usepackage{graphicx} 
\usepackage{natbib}

\usepackage{xcolor}
\usepackage{tabularx} 
\usepackage{amsthm}
\usepackage{amsfonts}
\usepackage{amsmath}

\usepackage{amssymb}
\usepackage{mathrsfs}
\usepackage{multirow}
\usepackage{bbm}
\usepackage{wrapfig}

\newtheorem{theorem}{Theorem}

\newtheorem{corollary}{Corollary}

\newtheorem{lemma}{Lemma}
\newcounter{assump}
\newtheorem{assumption}[assump]{Assumption}

\usepackage{enumitem}
\usepackage{wrapfig}

\newcommand{\beq}{\begin{equation}}

\newcommand{\eeq}{\end{equation}}
\newcommand{\beqs}{\begin{equation*}}
\newcommand{\eeqs}{\end{equation*}}

\usepackage[framemethod=TikZ]{mdframed}
\usepackage{amsthm}
\newcounter{theo}[section] 
\renewcommand{\thetheo}{\arabic{section}.\arabic{theo}}

\newcounter{lem}[section] 
\renewcommand{\thelem}{\arabic{section}.\arabic{lem}}

\newcounter{prf}[section]
\renewcommand{\theprf}{\arabic{section}.\arabic{prf}}

\newcounter{cor}[section]
\renewcommand{\thecor}{\arabic{section}.\arabic{cor}}

\newcounter{prop}[section] 
\renewcommand{\thelem}{\arabic{section}.\arabic{prop}}

\usepackage{caption}

\usepackage{hyperref}

\newcommand\extrafootertext[1]{%
    \bgroup
    \renewcommand\thefootnote{\fnsymbol{footnote}}%
    \renewcommand\thempfootnote{\fnsymbol{mpfootnote}}%
    \footnotetext[0]{#1}%
    \egroup
}

\author[*1]{Aldo Pacchiano}
\author[*2]{Philip Ball}
\author[*2]{Jack Parker-Holder}
\author[3]{Krzysztof Choromanski}
\author[2]{Stephen Roberts}

\affil[1]{%
    UC Berkeley
}
\affil[2]{%
    University of Oxford
}
\affil[3]{%
    Google Brain Robotics
}
\affil[*]{%
    Equal Contribution.
}

\begin{document}

\title{Towards Tractable Optimism in Model-Based Reinforcement Learning}
\maketitle




%

\begin{abstract}

The principle of optimism in the face of uncertainty is prevalent throughout sequential decision making problems such as multi-armed bandits and reinforcement learning (RL). To be successful, an optimistic RL algorithm must over-estimate the true value function (optimism) but not by so much that it is inaccurate (estimation error). In the tabular setting, many state-of-the-art methods produce the required optimism through approaches which are intractable when scaling to deep RL. We re-interpret these scalable optimistic model-based algorithms as solving a tractable noise augmented MDP. This formulation achieves a competitive regret bound: $\tilde{\mathcal{O}}( |\mathcal{S}|H\sqrt{|\mathcal{A}| T } )$ when augmenting using Gaussian noise, where $T$ is the total number of environment steps. We also explore how this trade-off changes in the deep RL setting, where we show empirically that estimation error is significantly more troublesome. However, we also show that if this error is reduced, optimistic model-based RL algorithms can match state-of-the-art performance in continuous control problems.

\end{abstract}
\extrafootertext{\small{Correspondence to: \texttt{pacchiano@berkeley.edu}, \texttt{\{ball, jackph\}@robots.ox.ac.uk}}}

\section{Introduction}
\label{sec:intro}

Reinforcement Learning (RL, \cite{suttonbarto}) considers the problem of an agent taking sequential actions in an uncertain environment to maximize some notion of reward. Model-based reinforcement learning (MBRL) algorithms typically approach this problem by building a ``world model'' \citep{dyna}, which can be used to simulate the true environment. This facilitates efficient learning, since the agent no longer needs to query to true environment for experience, and instead plans in the world model. In order to learn a world model that accurately represents the dynamics of the environment, the agent must collect data that is rich in experiences \citep{pathak_wm}. However, for faster convergence, data collection must also be performed efficiently, wasting as few samples as possible \citep{rp1}. Thus, the effectiveness of MBRL algorithms hinges on the exploration-exploitation dilemma.

This dilemma has been studied extensively in the tabular RL setting, which considers Markov Decision Processes (MDPs) with finite states and actions. \textit{Optimism in the face of uncertainty} (OFU) \citep{audibert, kocsis} is a principle that emerged first from the  Multi-Arm Bandit literature, where actions having both large expected rewards (exploitation) and high uncertainty (exploration) are prioritized. OFU is a crucial component of several state-of-the-art algorithms in this setting \citep{silver}, although its success has thus far failed to scale to larger settings. 


However, in the field of \emph{deep} RL, many of these theoretical advances have been overlooked in favor of heuristics \citep{burda}, or simple dithering based approaches for exploration \citep{mnih-atari-2013}. There are two potential reasons for this. First, many of the theoretically motivated OFU algorithms are intractable in larger settings. For example, UCRL2 \citep{ucrl2} a canonical optimistic RL algorithm, requires the computation of an analytic uncertainty envelope around the MDP, which is infeasible for continuous MDPs. Despite its many extensions \citep{filippi, ucrl2,fruit, osband, bartlett_1, tossou}, none address generalizing the techniques to continuous (or even large discrete) MDPs.

Second, OFU algorithms must strike a fine balance in what we call the \emph{Optimism Decomposition}. That is, they need to be optimistic enough to upper bound the true value function, while maintaining low estimation error. Theoretically motivated OFU algorithms predominantly focus on the prior. However, when moving to the deep RL setting, several sources of noise make estimation error a thorn in the side of optimistic approaches. We show that an optimistic algorithm can fixate on exploiting the least accurate models, which causes the majority of experience the agent learns from to be worthless, or even harmful for performance.

In this paper we seek to address both of these issues, paving the way for OFU-inspired algorithms to gain prominence in the deep RL setting. We make two contributions:

\textbf{Making provably efficient algorithms tractable} Our first contribution is to introduce a new perspective on existing tabular RL algorithms such as UCRL2. We show that a comparable regret bound can be achieved by being optimistic with respect to a \emph{noise augmented} MDP, where the noise is proportional to the amount of data collected during learning. We propose several mechanisms to inject noise, including count-scaled Gaussian noise and the variance from a bootstrap mechanism. Since the latter technique is used in many prominent state-of-the-art deep MBRL algorithms \citep{METRPO,whentotrust, pets, rp1}, we have all the ingredients we need to scale to that paradigm. 

\textbf{Addressing model estimation error in the deep RL paradigm} We empirically explore the Optimism Decomposition in the deep RL setting, and introduce a new approach to reduce the likelihood that the weakest models will be exploited. We show that we can indeed produce optimism with low model error, and thus match state of the art MBRL performance. 

The rest of the paper is structured as follows: 1) We begin with background and related work, where we formally introduce the Optimism Decomposition; 2) In Section \ref{sec:optimism} we introduce noise augmented MDPs, and draw connections with existing algorithms; 3) We next provide our main theoretical results, followed by empirical verification in the tabular setting; 4) We rigorously evaluate the Optimism Decomposition in the deep RL setting, demonstrating the scalability of our approach; 5) We conclude and discuss some of the exciting future directions we hope to explore.

\section{Background and Related Work}
\label{sec:background}
In this paper we study a sequential interaction between a learner and a finite horizon MDP $\mathcal{M} = (\mathcal{S}, \mathcal{A}, P, H, r, P_0)$, where $\mathcal{S}$ denotes the state space, $\mathcal{A}$ the actions, $P$ its dynamics, $H$ its episode horizon, $r \in \mathbb{R}^{|\mathcal{S} | \times |\mathcal{A}|}$ the rewards and $P_0$ the initial state distribution. For any state action pair $(s,a)$, we call $r(s,a)$ their true reward, which we assume to be a random variable in $[0,1]$. $P$ represents the dynamics and defines the distribution over the next states, i.e., $s' \sim P(s,a)$ with probability $P(s,a,s')$. At the beginning of each round $k$, the learner computes a policy $\pi_k$ which it uses to collect rewards and transition tuples in $\mathcal{M}$, for a total of $H$ steps. We use $k$ to denote the episode number and $h$ to index a timestep within an episode. 

Since we do not know the true reward nor dynamics, we must instead approximate these through estimators. For state action pair $(s,a)$, we denote the average reward estimator as $\hat{r}_k(s,a) \in \mathbb{R}$ and the average dynamics estimator\footnote{We write $\Delta_d$ to denote the $d-$dimensional simplex. } as $\hat{P}_k(s,a) \in\Delta_{|S|}$, where index $k$ refers to the episode. When training, the learner collects dynamics tuples during its interactions with $\mathcal{M}$, which in turn it uses during each round $t$ to produce a policy $\pi_k$ and an approximate MDP $\mathcal{M}_k = (\mathcal{S}, \mathcal{A}, \tilde{P}, H, \tilde{r}, P_0)$. In our theoretical results we will allow $\tilde{P}(s,a)$ to be a signed measure whose entries do not sum to one. This is purely a semantic devise, rendering the exposition of our work easier and more general, and in no way affects the feasibility of our algorithms and arguments. 


For any policy $\pi$, let $V(\pi)$ be the (scalar) value of $\pi$ and let $\tilde{V}_k(\pi)$ be the value of $\pi$ operating in the approximate MDP $\mathcal{M}_k$. We define $\mathbb{E}_{\pi}$ as the expectation under the dynamics of the true MDP $\mathcal{M}$ and using policy $\pi$ (analogously $\tilde{E}_{\pi}$ as the expectation under $\mathcal{M}_k$). The true and approximate value function for a policy $\pi$ are defined as follows:
\begin{small}
\begin{align*}
    V(\pi) = \mathbb{E}_{\pi}\left[ \sum_{h=0}^{H-1} r(s_h, a_h) \right],  \tilde{V}_k(\pi) = \tilde{\mathbb{E}}_{\pi}\left[ \sum_{h=0}^{H-1} \tilde{r}_k(s_h, a_h)\right].
\end{align*}
\end{small}
\begin{table*}[ht]
\begin{center}
\vspace{-2mm}
\caption{Prominent tabular RL algorithms and their noise augmented equivalents.}
\vspace{-2mm}
\scalebox{0.9}{
\begin{tabular}{ cccc } 
\toprule
\textbf{Algorithm}  & \textbf{Scalable?} & \textbf{Model/Value Based} & \textbf{Regret Bound} \\ \midrule
UCRL \citep{ucrl2}  & No  & Model Based & $\tilde{\mathcal{O}}(| \mathcal{S}| H\sqrt{| \mathcal{A}| T } )$ \\ 
UCBVI \citep{azar2017minimax}  & No  & Model Based & $\tilde{\mathcal{O}}(\sqrt{H | \mathcal{S}||\mathcal{A}| T } )$ \\ 
RLSVI \citep{osband2016generalization,russo2019worst}  & Yes  & Value Based & $\tilde{\mathcal{O}}(| \mathcal{S}| H^{5/2}\sqrt{ | \mathcal{S}||\mathcal{A}| T } )$ \\ 
Posterior Sampling \citep{psrl_osband}  & Yes  & Model Based & $\tilde{\mathcal{O}}(| \mathcal{S}| H\sqrt{ |\mathcal{A}| T } )$ \\ 
\midrule
Noise Augmented UCRL & Yes  & Model Based & $\tilde{\mathcal{O}}(| \mathcal{S}| H\sqrt{ | \mathcal{S}||\mathcal{A}| T } )$\\ 
Noise Augmented UCBVI & Yes  & Model Based & $\tilde{\mathcal{O}}(| \mathcal{S}| H \sqrt{ |\mathcal{A}| T } )$ \\ 
\bottomrule
\end{tabular}}
\vspace{-5mm}
\end{center}
\end{table*}
We will evaluate our method using \textit{regret}, the difference between the value of the optimal policy and the value from the policies it executed. Formally, in the episodic RL setting the regret of an agent using policies $\{ \pi_k \}_{k=1}^K$ is (where $K$ is number of episodes and $T=KH$):
\vspace{-.3cm}
\begin{align*}
    R(T) = \sum_{k=1}^K V(\pi^*) - V(\pi_k),
\end{align*}
where $\pi^*$ denotes the optimal policy for $\mathcal{M}$, and $V(\pi_k)$ is $\pi_k$ true value function. Furthermore, for each $h \in \{1, \cdots, H\}$ we call $\mathbf{V}^h(\pi) \in \mathbb{R}^{|\mathcal{S}|}$ the value vector satisfying $\mathbf{V}^h(\pi)[s] = \mathbb{E}_{\pi}\left[ \sum_{h' = h}^{H-1} r(s_{h'}, a_{h'}) | s_h = s\right] $, similarly we define $\tilde{\mathbf{V}}_k^h(\pi) \in \mathbb{R}^{|\mathcal{S}|}$ as $\tilde{\mathbf{V}}_k^h(\pi)[s] = \tilde{\mathbb{E}}_{\pi}\left[ \sum_{h' = h}^{H-1} \tilde{r}(s_{h'}, a_{h'}) | s_h = s \right]$ where $\mathbf{V}^{H}(\pi)[s] = \mathbf{0}$. Bold represents a vector-valued quantity.   

The principle of optimism in the face of uncertainty (OFU) is used to address the exploration-exploitation dilemma in sequential decision making processes by performing both simultaneously. In RL, ``model based" OFU algorithms \citep{ucrl2, fruit, tossou} proceed as follows: at the beginning of each episode $k$ a learner selects an approximate MDP $\mathcal{M}_k$ from a model cloud $\mathbf{{M}}_k$  and a policy $\pi_k$ whose approximate value function $\tilde{V}_k(\pi_k)$ is optimistic, that is, it overestimates the optimal policy's true value function $V(\pi^*)$. Our approach follows the same paradigm, but instead of using a continuum of models as in \citep{ucrl2,azar2017minimax} we allow $\mathbf{{M}}_k$ to be a discrete set (i.e. an ensemble). For OFU inspired algorithms we 
re-write $R(T)$ as:
\begin{small}
\begin{equation}\label{equation::regret_decomposition}
    R(T) = \underbrace{ \sum_{k=1}^K V(\pi^*) -\tilde{V}_k(\pi_k) }_{\mathrm{Optimism}}+ \underbrace{\sum_{k=1}^K \tilde{V}_k(\pi_k) - V(\pi_k)}_{\mathrm{Estimation Error}}.
\end{equation}
\end{small}
We refer to this as the \emph{Optimism Decomposition}, since it breaks down the regret into its' two major components. OFU algorithms must ensure that: the approximate value function is sufficiently optimistic (Optimism); and the estimated value function is not be too far from the true value function (Estimation Error). Balancing these two requirements forms the basis of all optimism based algorithms that satisfy provable regret guarantees. 

In this paper we aim to shed light on how to transfer the principle of optimism into the realm of model based deep RL with deep function approximation. To our knowledge we are the first to propose algorithms for the deep RL setting inspired by the optimism principle prevalent throughout the theoretical RL literature.

Two of the most prominent model-based OFU algorithms are UCRL2 and UCBVI. In the case of UCRL2, optimism is produced by analytically optimizing over the entire dynamics uncertainty set. It is easy to see that this is intractable beyond the tabular setting. In the case of UCBVI, optimism is produced by adding a bonus directly at the value function level. This is also intractable in the deep RL setting as it requires a count model over the visited states and actions.

Optimism beyond tabular models has been theoretically studied in \cite{kakade_2, jin_x} that showed the value of OFU where the MDP satisfies certain linearity properties but their practical impact has been limited.

Other methods which have successfully scaled from the tabular setting to deep RL are Posterior Sampling and RLSVI \citep{osband2016generalization, russo2019worst}. First we note that RLSVI is not model based in spirit and it certainly does not use a model ensemble. While RLSVI is a philosophically different algorithm to ours, they also propose the use of Gaussian noise perturbations and so it is closely related to our work. 

Our approach is also inspired by \cite{agrawal2017optimistic} and \cite{xu1001worst}. The parametric approach to posterior sampling studied in \cite{agrawal2017optimistic} can be easily analyzed under our framework, which can be seen as a generalization of the their posterior sampling algorithm. Crucially however, our method is simple to implement in the deep RL setting, opening the door to new scalable algorithms.

Next we show that optimism can be achieved by a simple noise augmentation procedure. This gives rise to provably efficient algorithms for tabular RL problems. We discuss variants of both UCRL and UCBVI which make use of this to simultaneously scale to deep RL while maintaining their theoretical guarantees.

\section{Algorithms}
\label{sec:optimism}

The aim of this section is to show that we can produce new versions of two popular OFU algorithms solely making use of \textbf{noise augmentation}. First, we focus on showing these noise augmented algorithms are theoretically competitive, before discussing practical implementations of our approach in the deep RL context. 

{\centering
\begin{minipage}{.99\linewidth}
    \begin{algorithm}[H]
    \small{
    \textbf{Input:} 
    Finite horizon MDP $\mathcal{M} = (\mathcal{S}, \mathcal{A}, P, H, r, P_0)$, Episodes $K$, Initial reward and dynamics augmentation noise distributions $\{ \mathbb{P}_1^r(s,a) \}_{s,a \in \mathcal{S}\times \mathcal{A}}$ and $\{ \mathbb{P}_1^P(s,a) \}_{s,a \in \mathcal{S}\times \mathcal{A}}$, sampling frequencies $M_r, M_P$.  \\
    \textbf{Initialize:} the transition and rewards data buffer $\mathcal{D}(s,a) = \emptyset$ for each $s,a \in \mathcal{S}\times \mathcal{A}$.\\
    \For{$k=1, \ldots , K-1$}{
    (1) $\forall s,a$ sample $M_P$ noise vectors  $\boldsymbol{\xi}_{k, P}^{(m)}(s,a) \sim \mathbb{P}_k^P(s,a)$. \\
    \textbf{UCRL2:}\\
    (2) For each $s,a$ sample $M_r$ noise values $\xi_k^{(m)}(s,a) \sim  \mathbb{P}_k^r(s,a)$.\\
    (3) Compute policy $\pi_k$ by running Noise Augmented Extended Value iteration as in Equation \ref{equation::noise_augmented_value_iteration}.\\
    \textbf{UCBVI}:\\    
        (2) Compute policy $\pi_k$ by running Noise Augmented Value iteration as in Equation \ref{equation::noise_augmented_value_iteration_UCBVI}.\\

(*) Execute policy $\pi_k$ for a length $H$ episode and update $\mathcal{D}$. Produce $\{ \mathbb{P}_{k+1}^r(s,a)\}_{s,a \in \mathcal{S} \times \mathcal{A}}$ and (if \textbf{UCRL}) $\{ \mathbb{P}_{k+1}^P(s,a) \}_{s,a \in \mathcal{S}\times \mathcal{A}}$.}
     }    
     \caption{Noise Augmented RL (NARL)}
    \label{alg::TOFU}
    \end{algorithm}
\end{minipage}
}

Although we present our results for the case of undiscounted episodic reinforcement learning problems, our results extend to the average reward setting with bounded diameter MDPs as in \citep{ucrl2,agrawal2017optimistic}. We are inspired by UCRL, but shift towards noise augmentation rather than an intractable model cloud. We thus call our approach \emph{Noise Augmented Reinforcemennt Learning} or NARL.

NARL initializes an empty data buffer of rewards and transitions $\mathcal{D}$. We denote by $N_k(s)$ the number of times state $s$ has been encountered in the algorithm's run up to the beginning of episode $k$ (before $\pi_k$ is executed). Similarly we call $N_k(s,a)$ the number of times the pair $(s,a)$ has been encountered up to the beginning of episode $k$, and let $N_k(s) = \sum_{a \in \mathcal{A} } N_k(s,a)$. 
We make the following assumption:
\begin{assumption}[Rewards]\label{assumption::subgaussian_rewards}
We assume the rewards are $1-$sub Gaussian with mean values in $[0,1]$.
\end{assumption}
\vspace{-.5cm}

\subsection{Concentration}\label{subsection::optimism}
We start by recalling the mean estimators $\{ \hat{r}_k(s,a) \}_{(s,a) \in \mathcal{S} \times \mathcal{A}}$ and $\{\hat{P}_k(s,a) \}_{(s,a) \in \mathcal{S} \times \mathcal{A}}$ concentrate around their true values. We make use of a time uniform concentration bound that leverages the theory of self normalization \citep{pena2008self,abbasi2011improved} to obtain the following:

\begin{lemma}[Lemma 1 of \cite{maillard2018upper}]\label{lemma::confidence_bounds} For all $(s,a) \in \mathcal{S} \times \mathcal{A}$:
\begin{small}
\begin{align*}
    \mathbb{P}\Big( \forall t \in \mathbb{N}\quad  | r(s,a) - \hat{r}_k(s,a)| \geq \beta_r(N_k(s,a),\delta') \Big) \leq \delta,& \\
\mathbb{P}\left(    \forall t \in \mathbb{N}\quad \| P(s,a)  - \hat{P}_k(s,a) \|_1 \geq  \beta_P(N_k(s,a), \delta')       \right) \leq \delta,& \\
 \text{s.t. }\beta_r(n, \delta') :\approx \sqrt{ \frac{\log\left( n/\delta' \right) }{n}},\quad\beta_P(n, \delta') :\approx  \sqrt{\frac{ |S| \log( n/\delta' )}{ n} }.&
\end{align*}
\end{small}
\end{lemma}
A more precise version of these bounds is stated in Appendix~\ref{section::auxiliary_results}.
Equipped with these bounds, for the rest of the paper we condition on the event:
\begin{small}
\begin{align*}
\mathcal{E} &:= \{ \forall k \in \mathbb{N}, \forall (s,a) \in \mathcal{S} \times \mathcal{A} ,\\
|&\quad r(s,a) - \hat{r}_k(s,a)| \leq \beta_r(N_k(s,a),\delta'),  \\
&\quad \| P(s,a)  - \hat{P}_k(s,a) \|_1 \leq  \beta_P(N_k(s,a), \delta')      \}    .
\end{align*}
\end{small}
If $\delta' = \frac{ \delta}{2|S||A|}$, Lemma \ref{lemma::confidence_bounds} implies $\mathbb{P}(\mathcal{E})\geq 1-\delta$.

At the beginning of the $k-$th episode the learner produces $M_r$ reward augmentation noise scalars $\xi_k^m(s,a)\sim \mathbb{P}^r_k(s,a)$ and possibly $M_P$ dynamics augmentation $| \mathcal{S} |$-dimensional noise vectors $\boldsymbol{\xi}^m(s,a) \sim \mathbb{P}_k^P(s,a)$, for each state action pair $(s,a) \in \mathcal{S} \times \mathcal{A}$. This notation will become clearer in the subsequent discussion.

\subsection{Noise Augmented UCRL}

We start by showing that an appropriate choice for the noise variables $\{\mathbb{P}_k^r(s,a)\}_{s,a\in \mathcal{S} \times \mathcal{A}}$ and $\{\mathbb{P}_k^P(s,a)\}_{s,a\in \mathcal{S} \times \mathcal{A}}$ yields an algorithm akin to UCRL2 and with provable regret guarantees.

Our main theoretical results of this section (Theorem \ref{theorem::main_gaussian}) states that in the tabular setting, if we set $\mathbb{P}_k^r(s,a) = \mathcal{N}(0,  \sigma^2_{t,r}(s,a)   ) $ and $\mathbb{P}_k^P(s,a) = \mathcal{N}(\mathbf{0}, \mathbb{I}_{|\mathcal{S}|} \sigma^2_{t,P}(s,a))$, for appropriate values of $\sigma^2_{t,r}(s,a)$ and $\sigma^2_{t,P}(s,a)$ we can obtain a regret guarantee of order $\tilde{\mathcal{O}}( |\mathcal{S}|H\sqrt{|\mathcal{S}||\mathcal{A}| T })$, which is competitive w.r.t. UCRL2 that achieves $\tilde{\mathcal{O}}( |\mathcal{S}|H\sqrt{|\mathcal{A}|T })$. These results can be extended beyond Gaussian noise augmentation provided the noise distributions satisfy quantifiable anticoncentration properties. For example, when using the dynamics noise given by posterior sampling of dynamics vectors, we recover the results of \cite{agrawal2017optimistic} in the episodic setting. Our results can be easily extended to the bounded diameter, average reward setting.

 Noise Augmented Extended Value Iteration (NAEVI) proceeds as follows: at the beginning of episode $k$ we compute a value function $\tilde{V}_k$ as:

\vspace{-2mm}
\begin{small}
\begin{align}
    \tilde{\mathbf{V}}_k^h( \pi_k )[s] &= \max_{a\in \mathcal{A}} \Big( \hat{r}_k(s,a)  +  \mathbb{E}_{s' \sim \hat{P}_k(s,a)}\left[  \tilde{\mathbf{V}}_k^{h+1}(s') \right]  \notag\\
    &\quad +\max_{m} \xi_k^{m} (s,a) + \max_{ m } \langle \boldsymbol{\xi}_k^{(m)}(s,a),   \tilde{\mathbf{V}}_k^{h+1} \rangle  \Big) \label{equation::noise_augmented_value_iteration}
\end{align}
\end{small}
\vspace{-2mm}
where $\mathrm{A} := \max_{ m} \xi_k^{m} (s,a)$ and $\mathrm{B} := \max_{ m } \langle \boldsymbol{\xi}_k^{(m)}(s,a),   \tilde{\mathbf{V}}_k^{h+1} \rangle$ represent the optimism bonuses for the \emph{present} and \emph{future} respectively. Many existing deep RL methods such as \cite{pseudocount1, pseudocount2, rnd} focus on adding bonuses that act like term $\mathrm{A}$. By adding term $\mathrm{B}$, Algorithm \ref{alg::TOFU} is able to take into account not only present but future rewards, and act optimistically according to them. In what follows, and for all states $(s,a)\in\mathcal{S}\times \mathcal{A}$, we will combine the noise values $\xi^{(m)}_k(s,a)$ and vectors $\boldsymbol{\xi}_k^{(m)}(s,a)$ with the average empirical reward $\hat{r}_k(s,a)$ and empirical dynamics $\hat{P}_k(s,a)$ and think of them as forming sample rewards and sample dynamics vectors:
\begin{align*}
    \tilde{r}_k^{(m)}(s,a) &= \hat{r}_k(s,a) + \xi^{(m)}_k(s,a) \\ 
    \tilde{P}_k^{(m)}(s,a) &= \hat{P}_k^{(m)}(s,a) + \boldsymbol{\xi}_k^{(m)}(s,a).
\end{align*}
Although $\tilde{P}_k^{(m)}(s,a)$ may not be a probability measure, for convenience we still treat it as a signed measure and write $\mathbb{E}_{s' \sim \tilde{P}^{(m)}_k(s,a)}\left[ \cdot \right]:= \langle \hat{P}_k(s,a) + \boldsymbol{\xi}_k^{(m)}(s,a), \cdot \rangle$. Let $\mathcal{M}_k = (\mathcal{S}, \mathcal{A}, \tilde{P}, H, \tilde{r}, P_0)$ be the approximate MDP resulting from collecting the maximizing rewards $\tilde{r}_k^{(m)}$ and dynamics vectors $\tilde{P}_k^{(m)}$ while executing NAEVI. In other words, for any state action pair $(s,a) \in \mathcal{S}\times\mathcal{A}$:
\begin{align}
    \tilde{r}_k(s,a) &= \max_{m =1, \cdots, M_r} \tilde{r}_k^m(s,a) \label{eq::r_tilde_definition} \\
    \tilde{P}_k(s,a) &= \arg\max_{\{\tilde{P}_k^{(m)}(s,a) \}_{k=1}^{M_P}} \langle \tilde{P}_k^{(m)}(s,a), \tilde{\mathbf{V}}_k^{h+1} \rangle. \notag
\end{align}
Our main result in this section is the following theorem:

\begin{theorem}\label{theorem::main_gaussian}
Let $\epsilon \in (0,1)$, $\delta = \frac{\epsilon}{4T}$, $M_r  \geq \frac{\log\left( \frac{2|\mathcal{S}||\mathcal{A}|H}{\delta}\right)}{3}$ and $M_P \geq 3+\frac{\log\left(\frac{2|\mathcal{A}|H}{\delta} \right)}{3}$. The regret $R(T)$ of UCRL Algorithm \ref{alg::TOFU} with Gaussian noise augmentation satisfies the following bound with probability at least $1-\epsilon$:
\vspace{-2mm}
\begin{equation*}
    R(T) \leq \tilde{\mathcal{O}}( |\mathcal{S}|H\sqrt{|\mathcal{S}||\mathcal{A}|T } )
\end{equation*}
$\tilde{O}$ hides logarithmic factors in $|\mathcal{A}|, |\mathcal{S}|, \epsilon$ and $T$ and:
\begin{align*}
    \xi_{k,r}^{(m)} \sim \mathcal{N}(0, \sigma_r^2), \quad\text{s.t } \sigma_r = 2 \beta_r(N_k(s,a),  \frac{ \delta}{2|S||A|}) \\
    \boldsymbol{\xi}_{k, P}^{(m)} \sim \mathcal{N}(\mathbf{0}, \sigma_P^2 \mathbb{I}), \quad \text{s.t }\sigma_P = 2 \beta_P\left( N_k(s,a), \frac{\delta}{|\mathcal{S}||\mathcal{A}|}\right)
\end{align*}
\end{theorem}
We remark these bounds are not optimal in $H$ and $S$, nevertheless, Theorem \ref{theorem::main_gaussian} shows this simple (and computationally scalable) noise augmented algorithm satisfies a regret guarantee. Our proof techniques are inspired but not the same as those of \cite{agrawal2017optimistic}. Our proofs proceed in two parts; returning to the Optimism Decomposition (Equation \ref{equation::regret_decomposition}), we deal with the Optimism and Estimation Error separately. The details of all proofs are in Appendix~\ref{section::optimism_appendix}.

\subsection{Noise Augmented UCBVI}

In this section we show that a simple modification of the previous algorithm can yield an even stronger regret guarantee. The chief insight is to note that under  Assumption~\ref{assumption::subgaussian_rewards}, the scale of the value function is at most $H$ and therefore instead of adding dynamics noise vectors $\boldsymbol{\xi}_k^{(m)}$ it is enough to simply scale up the variance of the reward noise components to ensure optimism at the value function level.
 

 Noise Augmented Value Iteration (NAVI) proceeds as follows: at the beginning of episode $k$ we compute a $Q-$function $\tilde{\mathbf{Q}}_k$ as:
\begin{small}
\begin{align}
    \tilde{\mathbf{Q}}_{k,h}(s,a) &=\min\Big( \tilde{\mathbf{Q}}_{k-1, h}(s,a), H, \tilde{r}_k(s,a) +  \notag\\
    &\quad\mathbb{E}_{s' \sim \hat{P}_k(s,a)} \left[  \tilde{V}_{k, h+1}(s,a)  \right]   \Big) \notag\\ 
    \tilde{\mathbf{V}}_{k,h}(s,a) &= \max_{a \in \mathcal{A}}   \tilde{\mathbf{Q}}_{k,h}(s,a) \label{equation::noise_augmented_value_iteration_UCBVI}.
\end{align}
\end{small}

Where $\tilde{r}_k(s,a)$ is defined as in equation \ref{eq::r_tilde_definition}. The policy executed at time $k$ by Noise Augmented UCBVI is the greedy policy w.r.t. $\tilde{\mathbf{Q}}_{k,h}(s,a)$. 

Our main result in this section is the following theorem:

\begin{theorem}\label{theorem::main_gaussian_UCBVI}
Let $\epsilon \in (0,1)$, $\delta = \frac{\epsilon}{4T}$ and $M_r  \geq \frac{\log\left( \frac{2|\mathcal{S}||\mathcal{A}|H}{\delta}\right)}{3}$. The regret $R(T)$ of UCBVI Algorithm \ref{alg::TOFU} with Gaussian noise augmentation satisfies the following bound with probability at least $1-\epsilon$:
\begin{equation*}
    R(T) \leq \tilde{\mathcal{O}}( |\mathcal{S}|H\sqrt{|\mathcal{A}|T } ) 
\end{equation*}
 $\tilde{O}$ hides logarithmic factors in $|\mathcal{A}|, |\mathcal{S}|, \epsilon$ and $T$ and:
\begin{equation*}
     \xi_{k, r}^{(m)} \sim \mathcal{N}(0,\sigma_r^2),\quad\quad \text{s.t. }\sigma_r = 2H \beta_r(N_k(s,a),  \frac{ \delta}{2|S||A|}).
\end{equation*}
\end{theorem}
\subsubsection{Boostrap Noise Augmentation}
Drawing inspiration from~\citep{boot_1, boot_2}, we introduce the following algorithm:
\begin{enumerate}
    \item Initialize $\mathcal{D}$ by adding  $2M_B$ tuples $\{ ( s, a, -1),(s, a, 1)\}_{i=1}^W$ to each state action pair $(s,a)\in \mathcal{S} \times \mathcal{A}$.
\end{enumerate}

Build $\mathbb{P}_k^r(s,a)$ via the following procedure:
\begin{enumerate}
    \item[2] For each $\{ (s^{(k)}_h,a^{(k)}_h)\}_{h=1}^H$ encountered during step (*) of Algorithm~\ref{alg::TOFU}, add $(s^{(k)}_h, a^{(k)}_h, r_h^{(k)})$ and $2M_B$ extra tuples $\{ ( s^{(k)}_h, a^{(k)}_h, -1),(s^{(k)}_h, a^{(k)}_h, 1)\}_{i=1}^{2M_B}$  to the data buffer $\mathcal{D}$. 
    \item[3] For all $(s,a) \in \mathcal{S} \times \mathcal{A}$, compute the empirical rewards $\{\hat{r}_k^{(m)}(s,a)\}_{m=1}^{M_r}$ by boostrap sampling with replacement from the data buffer $\mathcal{D}(s,a)$ with probability parameter $1/2$:
    \begin{equation*}
    \hat{r}^{(m)}_k(s,a) =  \frac{1}{\sum_{i=1}^{|\mathcal{D}_k(s,a)|} x_i }\sum_{i =1}^{ |\mathcal{D}_k(s,a)|} x_i r_i(s,a).  
    \end{equation*}
    Where $x_i$ are all i.i.d. Bernoulli random variables with parameter $\frac{1}{2}$ and $r_i(s,a)$ are the reward samples in the data buffer $\mathcal{D}(s,a)$. The value of $\tilde{r}_k(s,a)$ is again computed via equation \ref{eq::r_tilde_definition}. 
\end{enumerate}
Similar to RLSVI this algorithm doesn't need to maintain visitation counts. 
The following theorem holds:

\begin{theorem}
\label{theorem::main_bootsrap_UCBVI}
Let $\epsilon \in (0,1)$ and $\delta = \frac{\epsilon}{4T}$, $M_B=H\log(T)$ and $M_r  \geq \frac{\log\left( \frac{2|\mathcal{S}||\mathcal{A}|H}{\delta}\right)}{3}$. The regret $R(T)$ of Algorithm \ref{alg::TOFU} with Boostrap noise augmentation satisfies the following bound with probability at least $1-\epsilon$:
\begin{equation*}
    R(T) \leq \tilde{\mathcal{O}}( |\mathcal{S}|H\sqrt{|\mathcal{A}|T } ) 
\end{equation*}
 $\tilde{O}$ hides logarithmic factors in $|\mathcal{A}|, |\mathcal{S}|, \epsilon$ and $T$
\end{theorem}

\section{Anti-concentration and Optimism}

The fundamental principle behind our bounds is that noise injection gives rise to optimism. In order to show this we rely on anti-concentration properties of the noise augmentation distributions. For the sake of simplicity we present simple results regarding Gaussian noise variables, their anti-concentration properties and one-step optimism. More nuanced results extending the discussion to Bootstrap sampling are in Appendix~\ref{appendix::boostrap_optimism}.

\textbf{Benign variance.} We start by showing that whenever the noise is Gaussian and has an appropriate variance, with a constant probability each of the noise perturbed reward estimators $\tilde{r}_k^{(m)}$ is at least as large as the empirical mean, plus the confidence radius $\beta_r( N_k(s,a), \frac{\delta}{|\mathcal{S} | |\mathcal{A}|}) $. We can boost this probability by setting $M_r$ to be sufficiently large. The main ingredient behind this proof is the following Gaussian anti-concentration result:

\begin{lemma}\label{lemma::gaussian_lower_bound}
Lower bound on Gaussian density $\mathcal{N}(\mu, \sigma^2)$:
\begin{equation}\label{equation::lower_bound_gaussian_density}
 \mathbb{P}\left (X-\mu>t \right) \ge \frac{1}{\sqrt{2\pi}}\frac{\sigma t}{t^2+\sigma^2} e^{-\frac{t^2}{2\sigma^2}}.
\end{equation}
\end{lemma}


Using Lemma \ref{lemma::gaussian_lower_bound} we can show that as long as the standard deviation of $\xi_k^{(m)}(s,a)$ is set to the right value, $\tilde{r}_k^{(m)}(s,a) $ overestimates the true reward $r(s,a)$ with constant probability. 

\begin{lemma}\label{lemma::anticoncentration_rewards_gaussian}
 Let $(s,a) \in \mathcal{S} \times \mathcal{A}$. If $\tilde{r}^{(m)}_k(s,a) \sim \hat{r}_k(s,a) + \mathcal{N}(0, \sigma^2)$ for $\sigma = 2 \beta_r(N_k(s,a),  \frac{ \delta}{2|S||A|}) $ then:
 \vspace{-2mm}
\begin{equation}
    \mathbb{P}(\tilde{r}^{(m)}_k(s,a ) \geq r(s,a)  | \mathcal{E}) \geq \frac{1}{10}.
\end{equation}
\end{lemma}

\vspace{-2mm}
Lemma \ref{lemma::anticoncentration_rewards_gaussian} implies that with constant probability the values $\tilde{r}_k^{(m)}(s,a)$ are an overestimate of the true rewards. It is also possible to show that despite this property, $\tilde{r}_k(s,a)$ remain very close to $\hat{r}_k(s,a)$ and therefore to $r(s,a)$. In summary:
\begin{corollary}\label{corollary::reward_augmentation}
The sampled rewards $\tilde{r}_k(s,a)$ are optimistic:
\begin{small}
\begin{equation}\label{equation::rewards_optimism}
\mathbb{P}\left(     \tilde{r}_k(s,a) = \max_{m=1, \cdots, M_r} \tilde{r}^{(m)}_k(s,a) \geq  r(s,a)   \Big| \mathcal{E} \right) \geq 1-\left(\frac{1}{10}\right)^{M_r}
\end{equation}
\end{small}
while at the same time not being too far from the true rewards: 
\begin{small}
\begin{equation}\label{equation::reward_estimation}
    \mathbb{P}\left( | \tilde{r}_k(s,a) - r(s,a) | \geq 
    L \beta_r\left(N_k(s,a),  \frac{ \delta}{2|S||A|}\right)     \Big|\mathcal{E}     \right) \leq \frac{\delta}{ |\mathcal{S}| | \mathcal{A}|}.
\end{equation}
\end{small}
Where $L = \left( 2\sqrt{ \log\left(\frac{4| \mathcal{S}| |\mathcal{A}| M_r }{\delta}\right)}+ 1\right)$.
\end{corollary}

Corollary \ref{corollary::reward_augmentation} shows the trade-offs when increasing the number of models in an ensemble: it increases the amount of optimism, at the expense of greater estimation error of the sample rewards. 
A similar statement can be made of the dynamics in UCRL, but we defer the details to the appendix. 

\textbf{Boostrap optimism} The necessary anti-concentration properties of the sampling distribution corresponding to Theorem~\ref{theorem::main_bootsrap_UCBVI} are explored in Appendix~\ref{appendix::boostrap_optimism}. 

\textbf{RLSVI Comparison} RLSVI's regret bound is of the order of $\mathcal{O}(H^3 S^{3/2}\sqrt{AK})$ while NARL-UCRL2 is of the order of $\mathcal{O}(HS^{3/2}\sqrt{AK})$ and NARL-UCBVI is of the order of $\mathcal{O}(HS\sqrt{AK})$. Removing an additional $\sqrt{S}$ factor may prove more challenging since it is akin to the extra $\sqrt{d}$ factor present in the worst case regret bounds for Thomspon sampling in Linear bandits. Our rates for noise augmented NARL-UCBVI are superior to RLSVI, and even better in its $H$ dependence than the latest regret bounds for the RLSVI setting (see~\cite{agrawal2020improved}). NARL and RLSVI are incomparable when moving into the function approximation regime since in this setting RLSVI will not  be a model based algorithm. We also want to remark that the existing works on RLSVI are purely theoretical works and therefore there is no empirical evidence in neither~\cite{russo2019worst} nor~\cite{agrawal2020improved} regarding its usefulness in a deep RL setting.

\textbf{Comparison to Thompson Sampling using Dirichlet prior:} As explained above, the objective of this work is not to be state of the art and get the optimal regret guarantees. Our dependence on $S$ should be compared not with this approach but with RLSVI. Although the use of a Dirichlet prior allows the authors to get a better dependence on $S$, the resulting algorithm is infeasible in the Deep RL paradigm. It is unclear what would the equivalent of maintaining such a prior be when making use of function approximation.

\vspace{-2mm}
\section{Tabular Exploration Experiments}
\label{sec:experiments}

\begin{figure*}[ht]
\vspace{-4mm}
    \centering\begin{minipage}{0.95\textwidth}
    \centering\subfigure[RiverSwim]{\includegraphics[width=.32\linewidth]{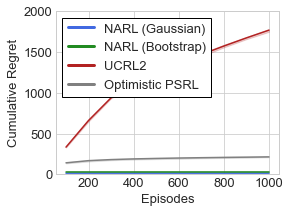}}
    \subfigure[Q-Values]{\includegraphics[width=.32\linewidth]{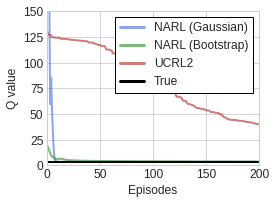}}
    \subfigure[Deep Sea]{\includegraphics[width=.3\linewidth]{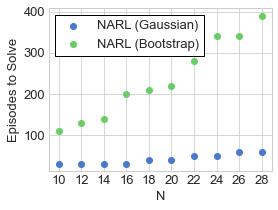}}     
    \vspace{-4mm}
    \caption{\small{Tabular RL experiments: a) RiverSwim b) Stochastic Chain c) Number of episodes to solve the Deep Sea task, the x-axis corresponds to increasing dimensionality.}}
    \label{figure:tabular}
    \end{minipage}
    \vspace{-6mm}
\end{figure*}

In this section we evaluate Noise Augmented UCRL (referred to as NARL) in the tabular setting, as is common for work in theoretical RL. We consider two implementations of NARL: (1) Gaussian, where we use use one model, but sample $M=10$ noise vectors from a Gaussian distribution, with variance $\frac{c}{N_k(s,a)}$ for a constant $c$, which we set to $1$, and (2) Bootstrap, where we maintain $M=10$ models, each having access to $50\%$ of the data.  We compare these against UCRL2 \citep{ucrl2} and Optimistic Posterior Sampling (OPSRL), using an open source implementation.\footnote{\url{https://github.com/iosband/TabulaRL}}

We begin with the RiverSwim environment \citep{mbie}, with $6$ states and en episode length of $20$. We repeat each experiment for $20$ seeds, to produce a median and IQR. In Fig.\ \ref{figure:tabular}(a) we see that both versions of NARL exhibit strong computational performance, while UCRL2 performs poorly. In Fig.\ \ref{figure:tabular}(b) we explore why this is the case, and plot the approximate value function for UCRL2 and NARL. We see that the weak performance for UCRL2 likely comes from over-estimation, i.e. being \textit{overly optimistic}, and it takes much longer to converge to the true value function. See the following link to run these experiments in a notebook: \url{https://bit.ly/3gVwsQF}.

We also explore the choice of noise augmentation, using the Deep Sea environment \citep{randomizedprior} from bsuite \citep{bsuite}. This experiment shows the ability to scale with increasing problem dimension. We used ten environments with $N=\{10,\dots,28\}$. As we see in Fig.\ \ref{figure:tabular}(c) NARL solves all ten tasks. Interestingly, the Gaussian method is best, indicating promise for this approach. 

Now we see NARL can compete empirically in the tabular setting, we next seek to demonstrate its' scalability in the deep RL paradigm. Note that other methods, such as UCRL2, are intractable beyond tabular environments. Meanwhile, the noise augmentation we propose uses ingredients commonly found in state-of-the-art deep MBRL methods, such as bootstrap ensembles.

\section{Optimism in Deep RL}
\label{subsec:practicaldeeprl}

Despite being a popular theoretical approach, optimism is not prevalent in the deep RL literature. The most prominent theoretically motivated deep RL algorithm is Bootstrapped DQN \citep{bootdqn} which is inspired by PSRL. However, it is well-known that Q-functions generally overestimate the true Q-values \citep{Thrun93issuesin}, therefore, many methods not using a lower bound (as used in TD3, \cite{td3}) may in fact be using an optimistic estimate. In recent times \citep{oac,Rashid2020Optimistic} present model-free approaches using optimistic policies to explore by shifting Q-values optimistically based on epistemic uncertainty. However, as far as we are aware, optimism is not widely used w.r.t the dynamics in deep \textbf{model based} RL. 

We know from our theoretical insights that an effective optimistic algorithm needs to balance the Optimism Decomposition. In the tabular setting we sought to add noise to boost the Optimism term, which led to \emph{too much} variance in the case of UCRL2. For deep RL, the dynamics are very different, as we add significant noise from function approximation with neural networks. In this section we introduce a scalable implementation of NARL, which builds on top of the state-of-the art continuous control (from states) MBRL algorithm. We also discuss the key factors influencing the Optimism Decomposition. 

\subsection{Noise via Bootstrapped Ensembles}\label{sec:noise_boostrapped_ensembles}

We implement our algorithm in by using an ensemble, as is common in existing state-of-the-art methods \citep{whentotrust, mbmpo, METRPO, pets, rp1}. For our implementation, we focus on \cite{whentotrust}, using probabilistic dynamics models \citep{nixweigend} and a Soft Actor Critic (SAC, \cite{sac, sac-v2}) agent learning inside the model. 

Dyna-style approaches \citep{dyna}, are particularly sensitive to model bias \citep{PILCO}, which often leads to catastrophic failure when policies are trained on inaccurate synthetic data. To prevent this, state-of-the-art methods such as MBPO randomly sample models from the ensemble to prevent the policy exploiting an individual (potentially biased) model. Rather than randomly sampling, we follow the Noise Augmented UCRL approach (Equation \ref{equation::noise_augmented_value_iteration}) and pass the same state-action tuple through each model, and select the highest predicted reward, and and assess which `hallucinated' next state has the highest expected return according to the critic of the policy thus providing us with an optimistic estimate of the transition dynamics. 

\begin{figure}[ht]
    \vspace{-2mm}
        \centering\begin{minipage}{0.5\textwidth}
        \centering\subfigure{\includegraphics[width=0.95\linewidth]{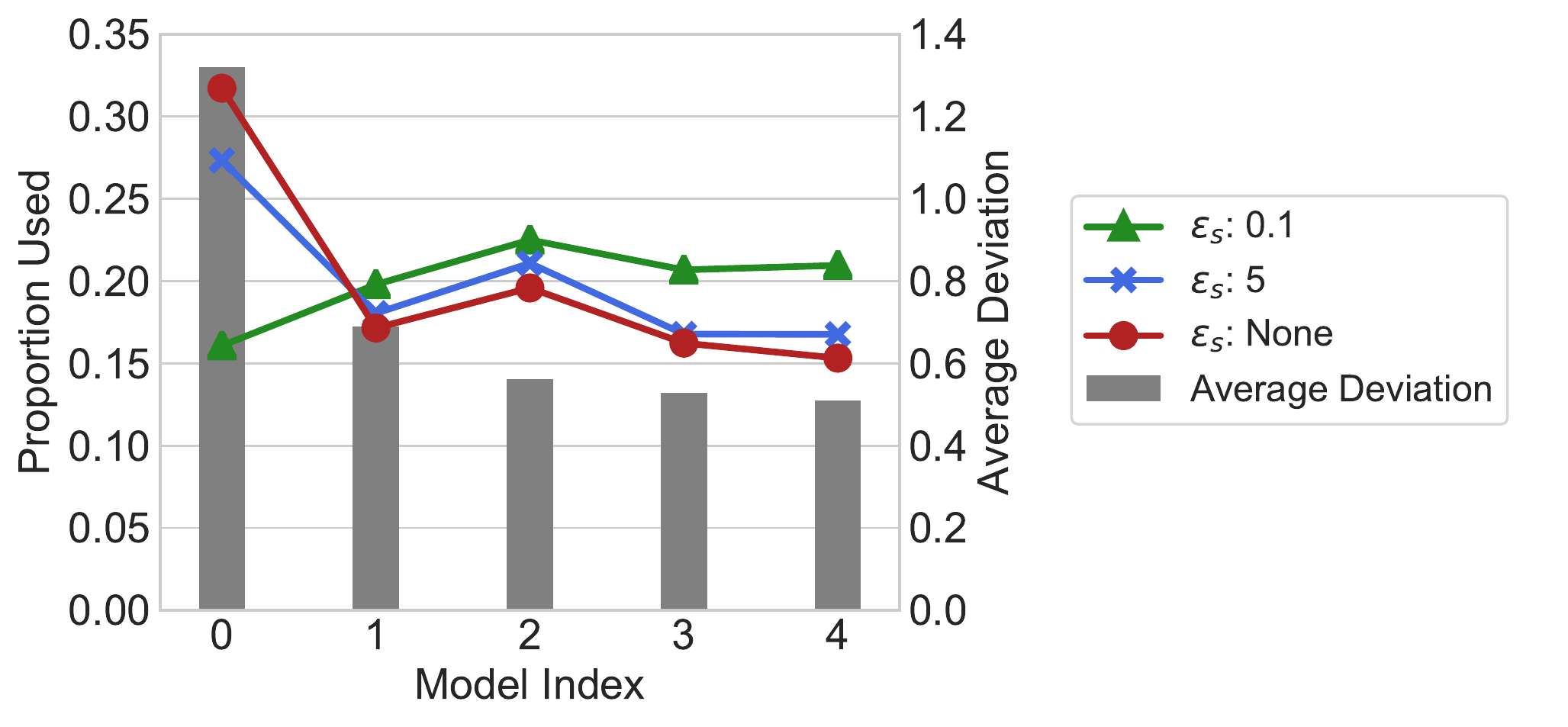}} 
        \end{minipage}
        \vspace{-6mm}
        \caption{\small{Ensemble member selection frequency under optimism.}}
        \vspace{-3mm}
        \label{fig:modelprop}
\end{figure}

\begin{figure*}[ht]
\vspace{-6mm}
    \centering\begin{minipage}{0.95\textwidth}
    \centering\subfigure[InvertedPendulum]{\includegraphics[width=.33\linewidth]{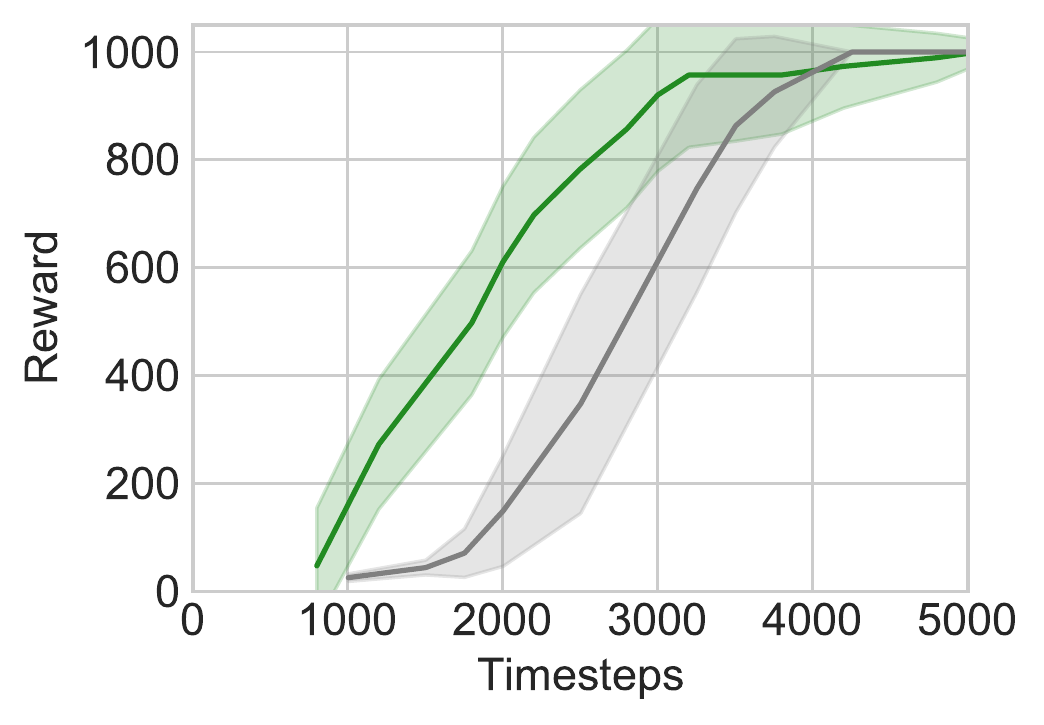}}
    \subfigure[Hopper]{\includegraphics[width=.3\linewidth]{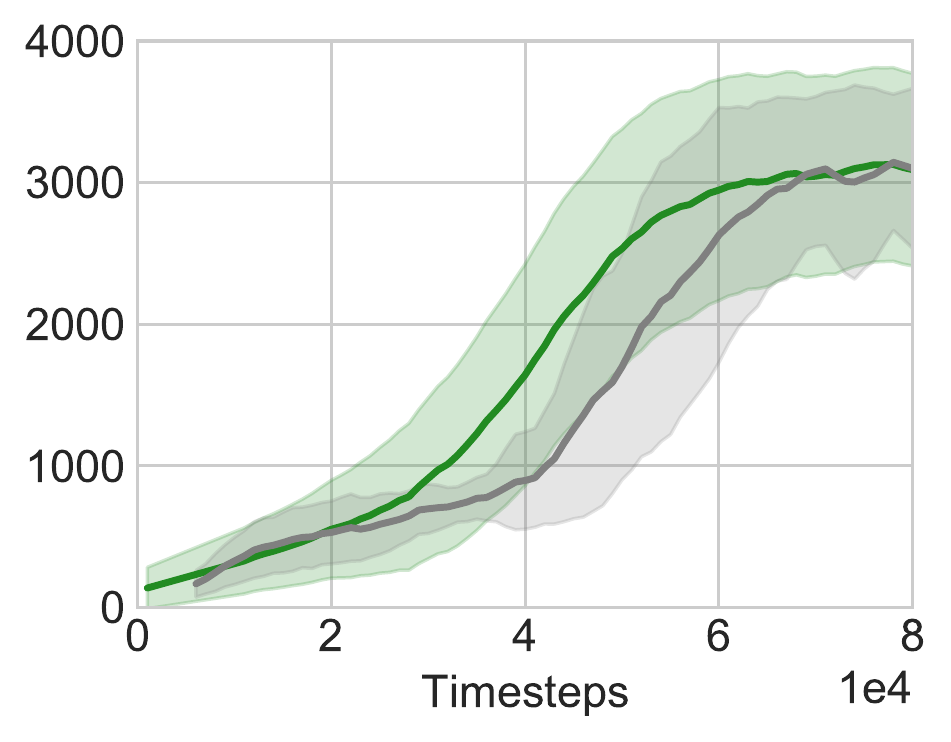}}
    \subfigure[HalfCheetah]{\includegraphics[width=.31\linewidth]{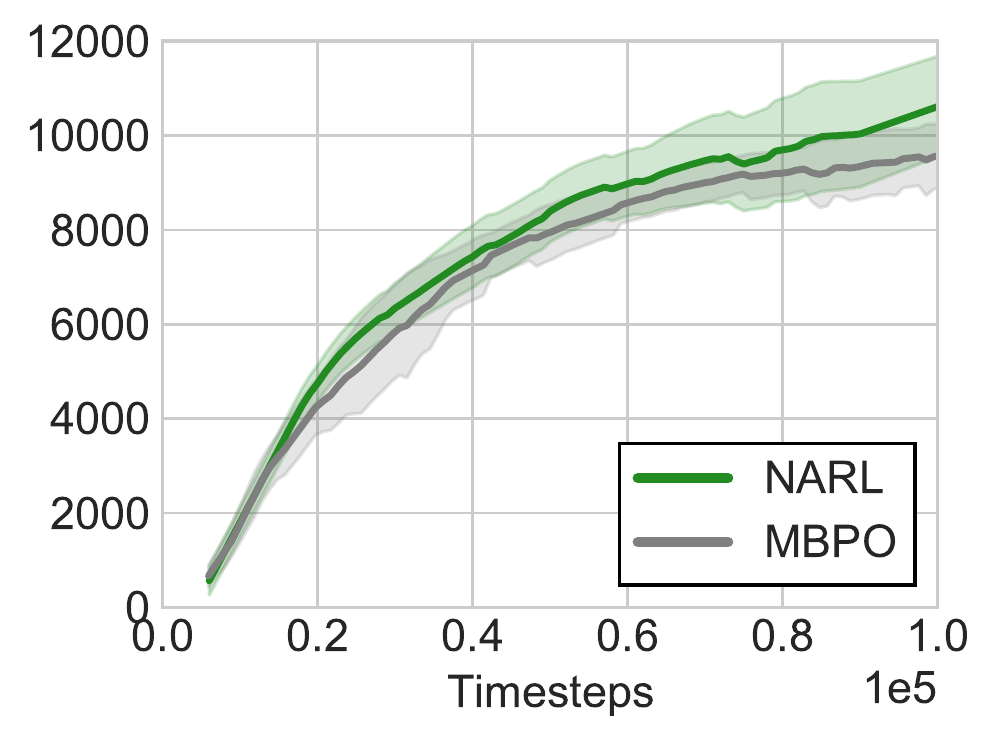}}
    \vspace{-4mm}
    \caption{\small{Curves show the mean $\pm$ one std for $\mathrm{InvertedPendulum}$ (left) and $\mathrm{Hopper}$ (right).}}
    \label{fig:gym}
    \end{minipage}
    \vspace{-6mm}
\end{figure*}

\textbf{You're only as good as your worst model:} However, in the deep RL setting, we introduce a significant amount of noise due to function approximation with neural networks. It has been observed in practice that this variance is sufficient to induce optimism \citep{bootdqn}. A key consideration is the tendency for optimism to select the \emph{individual} model with highest variance, resulting in over-exploitation of the least accurate models. In Fig. \ref{fig:modelprop} we demonstrate this phenomenon, by training an ensemble of models (ordered here in increasing validation accuracy) and comparing the proportion each model was selected (red) against the average distance from the mean of the next state estimates (grey); we observe that these quantities are positively correlated. Thus, the optimistic approach selects (and exploits) the least accurate models.

\textbf{Reducing Estimation Error} To balance the Optimism Decomposition in deep RL, we must focus on Estimation Error. We introduce a ``Model Radius Constraint" $\epsilon_M$: we calculate the empirical mean ($\mu_M$) of the expected returns, and exclude models that fall outside the permissible model sphere (defined as $\mu_M \pm \epsilon_M$) as being overly optimistic. Interestingly, an undocumented feature in MBPO \citep{whentotrust} that mirrors this is the idea of maintaining a subset of ``elite" models. Briefly, even though $K$ models are trained and maintained, in reality the top $E$ models are used for rollouts, where $E \leq K$. Even though models are sampled randomly in this approach, there is still a chance that ``exploitable" samples are generated by these poorer performing models.

The introduction of $\epsilon_M$ allows us to reduce Estimation Error from optimism. We see in Fig. \ref{fig:modelprop} that a wide radius ($\epsilon_M=5$, blue) has a small impact on reducing usage of the worst model. However, when we set a small radius ($\epsilon_M=0.1$, green), the models are selected almost uniformly. Details are in Appendix \ref{appendix:espi}. 

\subsection{Deep RL for Continuous Control}
\label{subsec:deeprlexp}

Now we evaluate the deep RL implementation of NARL. We focus on the $\mathrm{InvertedPendulum}$ task, as it is the simplest continuous environment and thus allows us to perform rigorous ablation studies. We run ten seeds for a variety of configurations, selecting the number of models $M$ from $\{3,5,10\}$ and $\epsilon_M$ from $\{0.1,5,\mathrm{None}\}$. These two parameters trade-off the amount of variance in the ensemble. Having more models means more noise. In addition, having a smaller $\epsilon_M$ will reduce variance. The results are presented in Table \ref{table:navi_results}.

Interestingly, we see strong evidence for our hypothesis that \emph{too much} variance is a problem in the deep RL setting. This results in the phenomenon whereby having fewer models (e.g. $M=3$) actually gives better performance, which has an added benefit of reduced computational cost. This is in contrast to methods based on random ensemble sampling \citep{METRPO,bootdqn}, where performance typically increases with the number of models. When using more models, the smaller model radius ($\epsilon_M$) is crucial. 

\vspace{-2mm}
\begin{table}[H]
\begin{minipage}[b]{0.99\linewidth}
\centering
    \caption{The mean number of timesteps to solve the $\mathrm{InvertedPendulum}$ task, with standard deviations.}
    \vspace{-3mm}
\scalebox{0.95}{
    \begin{tabular}{l*{4}{c}r}
    \toprule
    $\epsilon_M$ & 3 & 5 & 10  \\
    \midrule
    0.1 & 1850 $\pm166$  & 2350 $\pm$300  & 3300 $\pm$ 953  \\
    5 & 2000$\pm274$  & 2575 $\pm$251 & 3225 $\pm$675 \\
    $\mathrm{None}$ & 2000 $\pm$353 & 2775 $\pm$467 & 5850 $\pm$2037 \\
    \bottomrule
    \\
\end{tabular}}%
\vspace{-7mm}
\label{table:navi_results}
\end{minipage}
\end{table}

In Fig. \ref{fig:gym} we compare NARL against the publicly released data from  MBPO \citep{whentotrust} on the $\mathrm{InvertedPendulum}$, $\mathrm{Hopper}$ and $\mathrm{HalfCheetah}$ environments. For $\mathrm{InvertedPendulum}$ we show the results with $M=3$ and  $\epsilon_M=0.1$, the strongest result from Table \ref{table:navi_results}. With these parameters selected appropriately, we get meaningful gains against a very strong baseline, using an almost identical implementation aside from the model selection and $\epsilon_M$. For the larger $\mathrm{Hopper}$ and $\mathrm{HalfCheetah}$ tasks, we also used $M=3$ and selected $\epsilon_M$ from $\{0.1,0.5\}$. Again, we are able to perform favorably vs. MBPO, demonstrating the potential for our approach to scale to larger environments. This performance comes despite using over $50\%$ fewer models than MBPO ($3$ models vs. $7$).

We do not claim these results are state of the art, but highlight the design choices considered when using optimism for deep MBRL. In these settings we have been able to show that if variance can be controlled (e.g. by using $\epsilon_M$) then optimism can perform comparably well with the best random-sampling method.

\section{Conclusion and Future Work}
\label{sec:conclusion}

We introduce a new perspective on optimistic model-based reinforcement learning (RL) algorithms, using noise augmented MDPs. Our approach, NARL, achieves comparable regret bounds in the tabular setting, while crucially making use of mechanisms which naturally occur in the deep RL setting. As such, NARL offers the potential for designing scalable optimistic model-based reinforcement learning algorithms. We explored the key factors for successfully implementing NARL in the deep RL paradigm, and opened the door for the application of optimistic algorithms to deep model-based RL.

\bibliographystyle{plainnat}
\bibliography{refs}
\appendix
\newpage
\onecolumn
\section*{Appendix: Towards Tractable Optimism in Model-Based Reinforcement Learning}

Throughout Sections~\ref{section::optimism_appendix} and \ref{section::etimation_error_appendix} we make the following assumption:

\begin{assumption}\label{assumption::bounded_reward}
All rewards $r(s,a) \in [0,1]$ and all estimated rewards $\tilde{r}(s,a) \in [0,1]$. A simple clipping mechanism ensures the rewards $\tilde{r}(s,a)$ can be implemented to ensure this holds. 
\end{assumption}

\section{Optimism}\label{section::optimism_appendix}

\subsection{Auxiliary results}\label{section::auxiliary_results}
\begin{lemma}[Lemma 1 of \cite{maillard2018upper}]\label{lemma::confidence_bounds_appendix} For all $(s,a) \in \mathcal{S} \times \mathcal{A}$:
\begin{small}
\begin{align*}
    \mathbb{P}\Big( \forall t \in \mathbb{N}\quad  | r(s,a) - \hat{r}_k(s,a)| \geq \beta_r(N_k(s,a),\delta') \Big) \leq \delta, \\
    \text{ with }\beta_r(n, \delta') := \sqrt{ \frac{\log\left( 2\sqrt{n+1}/\delta' \right) }{n}}\\
\mathbb{P}\left(    \forall t \in \mathbb{N}\quad \| P(s,a)  - \hat{P}_k(s,a) \|_1 \geq  \beta_P(N_k(s,a), \delta')       \right) \leq \delta, \\
\text{ with } \beta_P(n, \delta') := \sqrt{\frac{ 4\log( \sqrt{n+1} \frac{2^{|S|}}{\delta'})}{ n} }.     
\end{align*}
\end{small}
\end{lemma}

\subsection{Gaussian Optimism}\label{section::gaussian_optimism}

\begin{lemma}\label{lemma::anticoncentration_rewards_gaussian_appendix}
 Let $(s,a) \in \mathcal{S} \times \mathcal{A}$. If $\tilde{r}^{(m)}_k(s,a) \sim \hat{r}_k(s,a) + \mathcal{N}(0, \sigma^2)$ for $\sigma = 2 \beta_r(N_k(s,a),  \frac{ \delta}{2|S||A|}) $ then:
 \vspace{-2mm}
\begin{equation}
    \mathbb{P}(\tilde{r}^{(m)}_k(s,a ) \geq r(s,a)  | \mathcal{E}) \geq \frac{1}{10}.
\end{equation}
\end{lemma}
\vspace{-5mm}
\begin{proof}
Since we are conditioning on $\mathcal{E}$, it follows that $\hat{r}_k(s,a) + \beta_r(N_k(s,a),  \frac{ \delta}{2|S||A|}) \geq r(s,a)$. 
As a consequence of Lemma \ref{lemma::gaussian_lower_bound}:
\begin{small}
\begin{align*}
    \mathbb{P}\left( \tilde{r}^{(m)}_k(s,a) \geq \hat{r}_k(s,a) + \beta_r(N_k(s,a),  \frac{ \delta}{2|S||A|})  \right) \geq  \\
    \frac{1}{\sqrt{2\pi}} \left(\frac{\sigma \beta_r(N_k(s,a),  \frac{ \delta}{2|S||A|})  }{\beta^2_r(N_k(s,a),  \frac{ \delta}{2|S||A|}) +  \sigma^2  }\right) e^{-\frac{\beta^2_r(N_k(s,a),  \frac{ \delta}{2|S||A|})  }{ 2\sigma^2} }
\end{align*}
\end{small}
Setting $\sigma= 2 \beta_r(N_k(s,a),  \frac{ \delta}{2|S||A|}) $ yields:
\begin{align*}
    \mathbb{P}\left( \tilde{r}^{(m)}_k(s,a) \geq \hat{r}_k(s,a) + \beta_r(N_k(s,a),  \frac{ \delta}{2|S||A|})  \right) \geq  \frac{1}{10}.
\end{align*}
After conditioning on $\mathcal{E}$, the result follows. 
\end{proof}

Lemma \ref{lemma::anticoncentration_rewards_gaussian} implies that with constant probability the values $\tilde{r}_k^{(m)}(s,a)$ are an overestimate of the true rewards. It is also possible to show that despite this property, $\tilde{r}_k(s,a)$ remain very close to $\hat{r}_k(s,a)$ and therefore to $r(s,a)$. Let $\sigma =  2 \beta_r(N_k(s,a),  \frac{ \delta}{2|S||A|})$ as in Lemma \ref{lemma::anticoncentration_rewards_gaussian}. Since $\tilde{r}^{(m)}_k(s,a) - \hat{r}_k(s,a) \sim \mathcal{N}( 0, \sigma^2)$, it follows that for all $t$ and all $(s,a) \in \mathcal{S} \times \mathcal{A}$:
\begin{small}
\begin{align}
\mathbb{P}\Big( |\tilde{r}^{(m)}_k(s,a) - \hat{r}_k(s,a) | \geq \label{equation::estimation_rewards}  \\
2\sqrt{ \log\left(\frac{4| \mathcal{S}| |\mathcal{A}| M_r }{\delta}\right)}  \beta_r\left(N_k(s,a),  \frac{ \delta}{2|S||A|}\right)  \Big)  \leq \frac{\delta}{|\mathcal{S}| | \mathcal{A}| M_r}. \notag
\end{align}
\end{small}
The probability of non-optimism decreases as the number of models increases, albeit at a logarithmic rate. Recall that while conditioning on $\mathcal{E}$, the confidence intervals are valid, and therefore $\tilde{r}_k(s,a)$ must also not be too far away from $\hat{r}_k(s,a)$ and therefore from $r(s,a)$. We can summarize the results of this section in the following Corollary:
\begin{corollary}\label{corollary::reward_augmentation_general_appendix}
The sampled rewards $\tilde{r}_k(s,a)$ are optimistic:
\begin{small}
\begin{equation}\label{equation::rewards_optimism_general_appendix}
\mathbb{P}\left(     \tilde{r}_k(s,a) = \max_{m=1, \cdots, M_r} \tilde{r}^{(m)}_k(s,a) \geq  r(s,a)   \Big| \mathcal{E} \right) \geq 1-\left(\frac{1}{10}\right)^{M_r}
\end{equation}
\end{small}
while at the same time not being too far from the true rewards: 
\begin{small}
\begin{equation}\label{equation::reward_estimation_general_appendix}
    \mathbb{P}\left( | \tilde{r}_k(s,a) - r(s,a) | \geq 
    L \beta_r\left(N_k(s,a),  \frac{ \delta}{2|S||A|}\right)     \Big|\mathcal{E}     \right) \leq \frac{\delta}{ |\mathcal{S}| | \mathcal{A}|}.
\end{equation}
\end{small}
Where $L = \left( 2\sqrt{ \log\left(\frac{4| \mathcal{S}| |\mathcal{A}| M_r }{\delta}\right)}+ 1\right)$.
\end{corollary}

Corollary \ref{corollary::reward_augmentation_general_appendix}
shows the trade-offs when increasing the number of models in an ensemble: it increases the amount of optimism, at the expense of greater estimation error of the sample rewards. 


\subsection{Bootsrap Optimism}\label{appendix::boostrap_optimism}

Although our proofs are based on injecting Gaussian noise into the rewards and dynamics mean estimators, it is possible to make use of any distribution or sampling scheme that guarantees enough optimism in these mean estimators while at the same time ensuring their estimation error converges to zero. Specifically, taking inspiration from \cite{boot_2} we propose the following reward augmentation technique. Every time a reward sample $r_k(s,a)$ is observed, we add additional $M_B$  fake reward samples $\{-1, 1\}$ into the buffer of rewards corresponding to state action pair $(s,a)$. For each state action pair, the number of fake rewards up to episode $k$ equals $2M_B N_k(s,a)$, while the number of real reward samples equals $N_k(s,a)$. This constant proportion between fake and real rewards is crucial in achieving optimism. In this case and while executing UCBVI-NARL we sample with replacement from the reward buffer corresponding to all state action pairs $(s,a) \in \mathcal{S} \times \mathcal{A}$. It can be seen, in a similar way as in \cite{boot_2}, this mechanism provides enough optimism for the reward signals. This yields the proof of the necessary boostrap sampling anticoncentration guarantees needed for concluding the proof of Theorem~\ref{theorem::main_bootsrap_UCBVI}, that from this point onward follows the same argument as the proof of Theorem~\ref{theorem::main_gaussian_UCBVI}. The necessary concentration properties follow also immediately from noting that the buffer samples are subgaussian.

\subsection{Rewards Data Augmentation.}\label{subsection::rewards_data_augmentation}

\paragraph{Dealing with $N_k(s,a) = 0$.} When $N_k(s,a) = 0$ we declare $\hat{r}_k(s,a) = 0$ and we let $\xi_k^{(m)}(s,a) \sim \mathcal{N}(0,1)$. In this case, and by Assumption \ref{assumption::bounded_reward}, simply by considering the probability of a sample to be larger than $1$, we conclude that in this edge case an equivalent version of Corollary~\ref{corollary::reward_augmentation_appendix} holds. This affects only mildly the constants in the definition of $M_r$.

We state and prove a slightly more general version of Lemma~\ref{lemma::anticoncentration_rewards_gaussian}:

\begin{lemma}\label{lemma::anticoncentration_rewards_gaussian_general_appendix}
 Let $(s,a) \in \mathcal{S} \times \mathcal{A}$. If $\tilde{r}^{(m)}_k(s,a) \sim \hat{r}_k(s,a) + \mathcal{N}(0, \sigma^2)$ for $\sigma = \gamma \beta_r(N_k(s,a),  \frac{ \delta}{2|S||A|}) $ then:
\begin{equation}
    \mathbb{P}(\tilde{r}^{(m)}_k(s,a ) \geq \hat{r}(s,a) + \beta_r( N_k(s,a), \frac{\delta}{|\mathcal{S} | |\mathcal{A}|}) ) \geq \frac{1}{\sqrt{2\pi } } \frac{\gamma}{1 + \gamma^2}e^{-\frac{ 1}{2\gamma^2}} \geq  \frac{1}{\sqrt{2\pi } } \frac{\gamma}{1 + \gamma^2}(1-\frac{1}{2\gamma^2})
\end{equation}
\end{lemma}

\begin{proof}
Since we are conditioning on $\mathcal{E}$, it follows that $\hat{r}_k(s,a) + \beta_r(N_k(s,a),  \frac{ \delta}{2|S||A|}) \geq r(s,a)$. 

As a consequence of Lemma \ref{lemma::gaussian_lower_bound}:
\begin{small}
\begin{align*}
    \mathbb{P}\left( \tilde{r}^{(m)}_k(s,a) \geq \hat{r}_k(s,a) + \beta_r(N_k(s,a),  \frac{ \delta}{2|S||A|})  \right) \geq  \frac{1}{\sqrt{2\pi}} \left(\frac{\sigma \beta_r(N_k(s,a),  \frac{ \delta}{2|S||A|})  }{\beta^2_r(N_k(s,a),  \frac{ \delta}{2|S||A|}) +  \sigma^2  }\right) e^{-\frac{\beta^2_r(N_k(s,a),  \frac{ \delta}{2|S||A|})  }{ 2\sigma^2} }
\end{align*}
\end{small}
Setting $\sigma= \gamma \beta_r(N_k(s,a),  \frac{ \delta}{2|S||A|}) $ yields:

\begin{align*}
    \mathbb{P}\left( \tilde{r}^{(m)}_k(s,a) \geq \hat{r}_k(s,a) + \beta_r(N_k(s,a),  \frac{ \delta}{2|S||A|})  \right) \geq  \frac{1}{\sqrt{2\pi } } \frac{\gamma}{1 + \gamma^2}e^{-\frac{ 1}{2\gamma^2}} \geq  \frac{1}{\sqrt{2\pi } } \frac{\gamma}{1 + \gamma^2}(1-\frac{1}{2\gamma^2})
\end{align*}
\end{proof}

And of Corollary~\ref{corollary::reward_augmentation}:

\begin{corollary}\label{corollary::reward_augmentation_appendix}
Let $p(\gamma) = \frac{1}{\sqrt{2\pi } } \frac{\gamma}{1 + \gamma^2}(1-\frac{1}{2\gamma^2})$.
\begin{equation}\label{equation::rewards_optimism_appendix}
\mathbb{P}\left(     \tilde{r}_k(s,a) = \max_{m=1, \cdots, M_r} \tilde{r}^{(m)}_k(s,a) \geq  r(s,a)   \right) \geq 1-\left(p(\gamma)\right)^{M_r}
\end{equation}
And conditioned on $\mathcal{E}$ and for all $(s,a) \in \mathcal{S}\times \mathcal{A}$:
\begin{equation}\label{equation::reward_estimation_appendix}
    \mathbb{P}\left( | \tilde{r}_k(s,a) - r(s,a) | \geq  \left( \sqrt{ \log\left(\frac{4| \mathcal{S}| |\mathcal{A}| M_r}{\delta}\right)}\gamma + 1\right) \beta_r\left(N_k(s,a),  \frac{ \delta}{2|S||A|}\right)          \right) \leq \frac{\delta}{ |\mathcal{S}| | \mathcal{A}|}
\end{equation}
\end{corollary}

\subsection{Dynamics Data Augmentation.} 

Following the same proof techniques as in Section~\ref{subsection::rewards_data_augmentation}, we can also show that for appropriate noise processes $\{\boldsymbol{\xi}_k^{(m)}(s,a)\}_{(s,a) \in \mathcal{S} \times \mathcal{A}}$, we can obtain an appropriate balance between optimism and estimation error, as we did for the rewards. In the case of a Gaussian noise process $\{\boldsymbol{\xi}_k^{(m)}(s,a)\}_{(s,a) \in \mathcal{S} \times \mathcal{A}}$, we can use anti concentration (Lemma~\ref{lemma::gaussian_lower_bound}) to show the following equivalent result to Lemma~\ref{lemma::anticoncentration_rewards_gaussian}.

\begin{lemma}\label{lemma::dynamics_gaussian_noise} Assume $|\mathcal{S}| \geq 2$ and let $(s,a) \in \mathcal{S} \times \mathcal{A} \times \mathcal{S}$. If $\tilde{P}^{(m)}_k(s,a, s') = \hat{P}_k(s,a, s') + \mathcal{N}(0, \sigma^2)$ for all $s'$ then and for $\sigma = 2 \beta_P\left( N_k(s,a), \frac{\delta}{|\mathcal{S}||\mathcal{A}|}\right) $, then for any fixed vector $\mathbf{v} \in \mathbb{R}^{\mathcal{S}}$ :
\begin{equation*}
    \mathbb{P}\left( \mathbb{E}_{ s' \sim \tilde{P}_k^{(m)}(s,a,s') } \left[ \mathbf{v}[s'] \right] \geq       \mathbb{E}_{ s' \sim P(s,a,s') } \left[ \mathbf{v}[s'] \right]      | \mathcal{E}   \right) \geq \frac{1}{9|\mathcal{S}|}. 
\end{equation*}
\end{lemma}
\vspace{-3mm}
The proof is in Appendix \ref{subsec:proof-dynamics-gaussian-noise}.

Since $\tilde{P}_k^{(m)}(s,a) - \hat{P}_k(s,a) \sim \mathcal{N}(0, I\sigma^2)$, it follows that for all $k$ and all $(s,a) \in \mathcal{S} \times \mathcal{A}$:
\begin{small}
\begin{equation}\label{equation::dynamics_estimation_error}
    \mathbb{P}\left(    \| \tilde{P}_k^{(m)}(s,a) - \hat{P}_k(s,a)\|_1 \geq 2\sqrt{|\mathcal{S}| \log\left(\frac{4| \mathcal{S}| |\mathcal{A}| M_P }{\delta}\right)}  \beta_P\left(N_k(s,a),  \frac{ \delta}{2|S||A|}\right)  \right)     \leq \frac{ \delta }{ | \mathcal{S} | | \mathcal{A} | M_P}.
\end{equation}
\end{small}
The extra $\sqrt{|\mathcal{S}|}$ not present in \ref{equation::estimation_rewards}, comes from bounding the $l_1$ norm of an $| \mathcal{S}|$-dimensional Gaussian vector and it is akin to the extra $\sqrt{d}$ in the regret of linear Thompson sampling~\cite{abeille2017linear}. Let $\mathbf{v} = \mathbf{V}^{h+1}(\pi^*) \in \mathbb{R}^{|\mathcal{S}|}$, the value vector of $\pi^*$ at $h+1$. Lemma \ref{lemma::dynamics_gaussian_noise} and Equation \ref{equation::dynamics_estimation_error} imply:
\begin{corollary}\label{corollary::dynamics_augmentation}
The sampled dynamics $\tilde{P}_k^{(m)}(s,a)$ are optimistic:
\begin{small}
\begin{equation*}
    \mathbb{P}\left(   \max_{ m=1, \cdots, M_P} \mathbb{E}_{s' \sim \tilde{P}^{(m)}_k(s,a,s')}\left[  \mathbf{V}^{h+1}(\pi^*)(s') \right] \geq \mathbb{E}_{s' \sim P(s,a,s')}\left[  \mathbf{V}^{h+1}(\pi^*)(s') \right] \right) \geq 1- \left(\frac{1}{9|\mathcal{S}|}\right)^{M_P}
\end{equation*}
\end{small}
while at the same time not being too far from the true dynamics: 
\begin{small}
\begin{equation}
\mathbb{P}\left(    \| \tilde{P}_k(s,a) - P(s,a)\|_1 \geq \left( 2\sqrt{|\mathcal{S}| \log\left(\frac{4| \mathcal{S}| |\mathcal{A}| M_P }{\delta}\right)} +1 \right)\beta_P\left(N_k(s,a),  \frac{ \delta}{2|S||A|}\right) \Big| \mathcal{E} \right)     \leq \frac{ \delta }{ | \mathcal{S} | | \mathcal{A} |  }.
\end{equation}
\end{small}
\end{corollary}
\vspace{-3mm}

Now we can proceed to show that as long as $M_P$ is chosen appropriately, we can ensure optimism holds with enough probability:

\begin{theorem}[Optimism]\label{theorem::optimism_appendix}
If $M_r  \geq \frac{\log\left( \frac{2|\mathcal{S}||\mathcal{A}|H}{\delta}\right)}{3}$ and $M_P \geq 3+\frac{\log\left(\frac{2|\mathcal{A}|H}{\delta} \right)}{3}$. Then with probability at least $1-2\delta$:
\begin{equation*}
    \tilde{V}_k( \pi_k  ) \geq V(\pi^*)
\end{equation*}
and therefore $\mathrm{I} \leq 0$ with probability at least $1-2\delta K$.
\end{theorem}
\vspace{-2mm}
The proof makes use of an inductive argument and can be found in Appendix~\ref{section::proof_optimism_main_theorem}.


\paragraph{Dealing with $N_k(s,a) = 0$.} When $N_k(s,a) = 0$ we declare $\hat{P}_k(s,a) =\frac{1}{|\mathcal{S}|} \mathbf{1}$ and we let $\boldsymbol{\xi}_k^{(m)}(s,a) \sim \mathcal{N}(0, \mathbb{I}_{|\mathcal{S}|})$. We can also make use of the alternative, $\hat{P}_k(s,a) =\mathbf{0}$ if we allow for $\hat{P}_k(s,a)$ to be a signed measure that is not a probability measure. Using these definitions we can easily derive a version of Corollary~\ref{corollary::dynamics_augmentation}. Taking this into account only adds a simple adjustment of the constants for the definition of $M_P$.

\subsubsection{Proof of Lemma~\ref{lemma::dynamics_gaussian_noise}}
\label{subsec:proof-dynamics-gaussian-noise}

\begin{proof}
Since conditioned on $\mathcal{E}$, $\| P(s,a) - \hat{P}_k(s,a)\|_1 \leq \beta_P( N_k(s,a), \frac{\delta}{|\mathcal{S}||\mathcal{A}|})$ and therefore: 
\begin{align*}
|\langle v, \hat{P}_k(s,a) \rangle  - \langle v, P(s,a) \rangle | &\leq \| P(s,a) - \hat{P}_k(s,a)\|_1\| v \|_\infty \\
&\leq \beta_P\left( N_k(s,a), \frac{\delta}{|\mathcal{S}||\mathcal{A}|}\right)\| v \|_\infty
\end{align*}
Let $\tilde{P}^{(m)}_k(s,a, s') -\hat{P}_k(s,a, s') = \xi \sim \mathcal{N}(0, \sigma^2)$. Notice that $\langle \tilde{P}_k^{(m)}(s,a,s'), v \rangle = \langle \tilde{P}_k^{(m)}(s,a,s'), v\rangle + \langle v, \xi \rangle$. And therefore $\langle v, \xi \rangle \sim \mathcal{N}( 0, \| v\|^2 \sigma^2 )$. Hence as a consequence of Lemma \ref{lemma::gaussian_lower_bound}:
\begin{align*}
    \mathbb{P}\left( \langle v, \xi \rangle \geq  \beta_P\left( N_k(s,a), \frac{\delta}{|\mathcal{S}||\mathcal{A}|}\right)\| v \|_\infty    \right) &\geq \frac{1}{\sqrt{2\pi}} \frac{\gamma \| v\|_2 \| v \|_\infty  }{ \gamma^2 \|v\|_2^2 + \| v \|^2_\infty } \exp\left(  - \frac{\| v \|^2_\infty  }{2 \gamma^2 \|v\|_2^2  }      \right) \\
    &\stackrel{(i)}{\geq} \frac{1}{\sqrt{2\pi}} \frac{\gamma \| v\|_2 \| v \|_\infty  }{ \gamma^2 \|v\|_2^2 + \| v \|^2_\infty } \exp\left(  - \frac{1 }{2 \gamma^2   }      \right)\\
    &\stackrel{(ii)}{\geq} \frac{1}{\sqrt{2\pi}} \frac{ \gamma}{S\gamma^2 +1}(1-\frac{1}{2\gamma^2})
\end{align*}
Inequality $(i)$ holds because $\| v \|_\infty \leq \|v\|_2$. Inequality $(ii)$ holds as a consequence of:
\begin{equation*}
    \gamma \| v \|_2 \| v \|_\infty \geq \gamma \| v \|^2_\infty = \frac{ \gamma}{S\gamma^2 + 1}\left(S\gamma^2 \|v\|_\infty^2 + \| v\|_\infty^2 \right) \geq  \frac{ \gamma}{S\gamma^2 + 1}\left(\gamma^2 \|v\|_2^2 + \| v\|_\infty^2 \right).
\end{equation*}

\end{proof}

\subsection{Proof of Optimism Theorem}\label{section::proof_optimism_main_theorem}

In this section we prove Theorem~\ref{theorem::optimism_appendix}. 

\begin{proof}
We proceed by induction. Notice that for all $s \in \mathcal{S}$:
\begin{equation*}
    \mathbf{V}^{H-1}(\pi^*)[s] = \max_{a \in \mathcal{A}} r(s,a)  
\end{equation*}
By Corollary~\ref{corollary::reward_augmentation}, if $M_r \geq \frac{\log\left( \frac{2|\mathcal{S}||\mathcal{A}|H}{\delta}\right)}{3}$, for each $(s,a) \in \mathcal{S} \times \mathcal{A}$ and with probability at least $1-\frac{\delta}{2|\mathcal{S}||\mathcal{A}| H}$:
\begin{equation}
    \tilde{r}_k(s,a) \geq r(s,a).
\end{equation}
Therefore, for any $s$, with probability at least $1-\frac{\delta}{2|\mathcal{S}|  H}$:
\begin{equation*}
    \tilde{\mathbf{V}}_k^{H-1}(\pi_k)[s] = \max_{a\in \mathcal{A}} \tilde{r}_k(s,a) \geq \max_{a \in \mathcal{A}} r(s,a)  = \mathbf{V}^{H-1}(\pi^*)[s]
\end{equation*}
And therefore it also holds with probability at least $1-\frac{\delta}{2H}$ and for all $s \in \mathcal{S}$ simultaneously.

We proceed by induction. Let's assume that for some $h+1 \leq H-1$ and for all $s$ and with probability at least $1-\delta(h+1)$ simultaneously for all $s\in \mathcal{S}$ it holds that $\tilde{\mathbf{V}}_k^{h+1}(\pi_k)[s] \geq \mathbf{V}^{h+1}(\pi^*)[s]$  for some value $\delta(h+1)$ dependent on $h$. Recall that by NAVI (Equation \ref{equation::noise_augmented_value_iteration}):
\begin{equation*}
    \tilde{\mathbf{V}}_k^{h}(\pi_k)[s] = \max_{a\in \mathcal{A}} \left( \tilde{r}_k(s,a)  +  \mathbb{E}_{s' \sim \tilde{P}_k(s,a)}\left[  \tilde{\mathbf{V}}_k^{h+1}(\pi_k)(s') \right]  \right) 
\end{equation*}
It follows that with probability at least $1-\delta(h+1)$ and simultaneously for all $s \in \mathcal{S}$:
\begin{align*}
    \tilde{\mathbf{V}}_k^{h}(\pi_k)[s] \geq \max_{a\in \mathcal{A}} \left( \tilde{r}_k(s,a)  +  \mathbb{E}_{s' \sim \tilde{P}_k(s,a)}\left[  \mathbf{V}^{h+1}(\pi^*)(s') \right]  \right) 
\end{align*}
Call this event $\mathcal{U}(h+1)$. Notice that for all $h'$, the value vector $\mathbf{V}^{h'}(\pi^*)$ is independent of $k$. If $M_P \geq \frac{\log\left(\frac{2|\mathcal{S}||\mathcal{A}|H}{\delta} \right)}{\log(9|\mathcal{S}|)} $, by Corollary~\ref{corollary::reward_augmentation} and Corollary~\ref{corollary::dynamics_augmentation}, for each $(s,a) \in \mathcal{S} \times \mathcal{A}$ and with probability at least $1-\frac{\delta}{| \mathcal{S}| | \mathcal{A}| H}$:
\begin{equation*}
    \tilde{r}(s,a) \geq r(s,a) \qquad \text{and}\qquad \mathbb{E}_{s' \sim \tilde{P}_k(s,a)}\left[  \mathbf{V}^{h+1}(\pi^*)(s') \right]  \geq \mathbb{E}_{s' \sim P(s,a)}\left[  \mathbf{V}^{h+1}(\pi^*)(s') \right] 
\end{equation*}
Therefore, a union bound implies that with probability at least $1-\frac{\delta}{H} - \mathbb{P}(\mathcal{U}^c(h+1)) = 1-\frac{\delta}{H} - \delta(h+1)$ and simultaneously for all $s \in \mathcal{S}$:
\begin{equation*}
       \tilde{\mathbf{V}}_k^{h}(\pi_k)[s]  \geq \max_{a\in \mathcal{A}} \left( r(s,a)  +  \mathbb{E}_{s' \sim P(s,a)}\left[  \mathbf{V}^{h+1}(\pi^*)(s') \right]  \right)  = \mathbf{V}^h(\pi^*)[s]
\end{equation*}
The RHS equality holds by optimality of $\pi^*$. This completes the induction step. Notice that we can set $\delta(H-1) = \frac{\delta}{2H}$. And that for all other $h$ we can define $\delta(h) = \delta(h)+ \frac{\delta}{H}$. Unrolling the induction until $h=0$ we can conclude that with probability at least $1-\delta$ and for all $s \in \mathcal{S}$:
\begin{equation*}
    \tilde{\mathbf{V}}_k^0(\pi_k)[s] \geq \mathbf{V}^0(\pi^*)[s]
\end{equation*}
Taking expectations w.r.t. $P_0$ concludes the proof by noting $\tilde{V}_k(\pi_k) = \mathbb{E}_{s \sim P_0}[\tilde{\mathbf{V}}_k^0(\pi_k)[s] ]$ and $V(\pi^*) = \mathbb{E}_{s\sim P_0} [\mathbf{V}^0(\pi^*)[s]] $. 

The second part of the statement follows by a simple union bound.

\end{proof}

\section{Estimation Error}\label{section::etimation_error_appendix}

The goal of this section is to bound term $\mathrm{II}$ of Equation \ref{equation::regret_decomposition}. Let's define an intermediate MDP $\hat{\mathcal{M}}_k$ corresponding to the MDP having $\mathcal{M}$'s true dynamics and using the rewards $\{ \tilde{r}(s,a) \}_{s,a \in \mathcal{S} \times \mathcal{A}}$. Similarly let's define an approximate value function $\hat{V}_k(\pi)$ which corresponds to the expected reward of $\pi$ on MDP $\hat{\mathcal{M}}_k$. Throughout this section we use the convention that whenever $N_k(s,a) = 0$, we instead use the value $1$ in the definition of the relevant confidence intervals.

Term $\mathrm{II}$ can be written as:

\begin{align*}
    \mathrm{II}  &= \sum_{k=1}^K \tilde{V}_k(\pi_k) - V(\pi_k) \\
    &= \underbrace{ \sum_{k=1}^K \tilde{V}_k(\pi_k)  - \hat{V}_k(\pi_k) }_{A} +\underbrace{ \sum_{k=1}^K \hat{V}_k
    (\pi_k)- V(\pi_k)}_{B}
\end{align*}

Throughout this section (bounds of terms $A$ and $B$) we condition on the event $\mathcal{E}$ defined as a result of Lemma~\ref{lemma::confidence_bounds}. Recall $\mathbb{P}(\mathcal{E}) \geq 1-\delta$.

\subsection{Bounding term \textbf{B}}


 Observe that the dynamics in $\mathcal{M}$ and $\hat{\mathcal{M}}_k$ are the same. The bound proceeds in two steps. 

First notice that by definition:

\begin{equation*}
    \hat{V}_k( \pi_k) = \mathbb{E}_{\pi_k}\left[ \sum_{h=0}^{H-1}  \tilde{r}_k(s_h, a_h) \right]
\end{equation*}

And therefore:

\begin{equation*}
    \hat{V}_k(\pi_k)  - V(\pi_k) = \mathbb{E}_{\pi_k}\left[ \tilde{r}_k(s_h, a_h) - r(s_h, a_h)      \right]
\end{equation*}

Let $\{ s_h^{(k)}, a_h^{(k)}\}_{h=1}^H$ be the (random) states and actions our algorithm executes at time $t$. Let $\mathcal{F}_{t-1}$ be the filtration corresponding to all the randomness in our process up to the beginning of episode $t$ (before the policy is executed). It follows that:

\begin{equation*}
   \hat{V}_k(\pi_k)  - V(\pi_k) = \mathbb{E}\left[ \sum_{h=1}^{H-1}  \tilde{r}_k( s_h^{(k) } , a_h^{(k)}) - r(s_h^{(k)}, a_h^{(k)}) | \mathcal{F}_{t-1}   \right] 
\end{equation*}

Let $X_k = \hat{V}_k(\pi_k)  - V(\pi_k) - \left( \sum_{h=1}^H  \tilde{r}_k( s_h^{(k) } , a_h^{(k)}) - r(s_h^{(k)}, a_h^{(k)}) \right)$. Let $Y_k = \sum_{\ell =1}^t X_k$ for all $t \geq 1$ and $Y_0  = 0$ It is easy to see $Y_k$ is a martingale satisfying a bonded differences assumption:

\begin{equation*}
    | X_k | \leq 4H
\end{equation*}

 Which holds since  $r$ and $\tilde{r}$ are both bounded by $1$. A simple application of the Azuma-Hoeffding\footnote{We use the following version of Azuma-Hoeffding: if $Y_k, t\geq 1$ is a martingale such that $|Y_k - Y_{k-1}| \leq d_k$ for all $t$ then for every $T \geq 1$ we have $\mathbb{P}\left(  Y_k \geq r \right) \leq \exp\left(-\frac{r^2}{2\sum_{k=1}^K d_k^2 } \right)$  } inequality for Martingales yields, for any $\delta \in (0,1)$ with probability at least $1-\delta$:

\begin{align*}
    Y_K \leq H\sqrt{2K  \log\left( \frac{1}{\delta}\right)} 
\end{align*}

This bound implies that with probability at least $1-\delta$:

\begin{equation*}
    B \leq H\sqrt{2K  \log\left( \frac{1}{\delta}\right)} + \sum_{k=1}^K  \left( \sum_{h=0}^{H-1}  \tilde{r}_k( s_h^{(k) } , a_h^{(k)}) - r(s_h^{(k)}, a_h^{(k)})\right)
\end{equation*}

Let's condition event $\mathcal{E}$. In this case:

\begin{equation*}
    \tilde{r}_k( s_h^{(k) } , a_h^{(k)}) - r(s_h^{(k)}, a_h^{(k)}) \leq \begin{cases} 1 & \text{if }N_k(s_h^{(k)}, a_h^{(k)}) = 0\\
    \sqrt{ \frac{2\log\left(\frac{2\sqrt{N_k(s^{(k)}_h,a^{(k)}_h)+1}}{\delta} \right) }{ N_k(s^{(k)}_h,a^{(k)}_h) }  } & \text{o.w.}
    \end{cases}
\end{equation*}

As a consequence of this:

\begin{align*}
    \sum_{k=1}^K \left( \sum_{h=0}^{H-1}  \tilde{r}_k( s_h^{(k) } , a_h^{(k)}) - r(s_h^{(k)}, a_h^{(k)})\right) &\leq 
     \sum_{(s, a) \in \mathcal{S}\times \mathcal{A}} \sum_{\ell=1}^{N_K(s,a)} \sqrt{ \frac{2\log\left(\frac{2\sqrt{\ell+1}}{\delta} \right) }{ \ell }  } +| \mathcal{S} | |\mathcal{A}|\\
    &\leq \sqrt{2\log\left(\frac{2\sqrt{KH}}{\delta}\right) }\sum_{(s, a) \in \mathcal{S}\times \mathcal{A}} 2\sqrt{N_K(s,a)}+| \mathcal{S} | |\mathcal{A}|
\end{align*}


The last inequality holds because for all $(s,a)$, it follows that $2\sqrt{N_k(s^{(k)}_h,a^{(k)}_h)+1} \leq 2\sqrt{KH}$

Since $\sum_{(s,a) \in \mathcal{S} \times \mathcal{A}} N_K(s,a) = KH$, and $\sqrt{ \cdot} $ is a concave function:

\begin{equation*}
    \sum_{(s, a) \in \mathcal{S}\times \mathcal{A}} \sqrt{N_K(s,a)} \leq \sqrt{|\mathcal{S} | |\mathcal{A} | KH}
\end{equation*}


Assembling these pieces together we can conclude that:

\begin{equation*}
    B \leq H\sqrt{2K  \log\left( \frac{1}{\delta}\right)} + 2\sqrt{2\log\left(\frac{2\sqrt{KH}}{\delta}\right) \mathcal{S}\mathcal{A} KH}  + | \mathcal{S} | |\mathcal{A}|
\end{equation*}


\subsection{Bounding term \textbf{A}}

This term is more challenging since although the rewards are the same, the dynamics are different. We introduce an extra bit of notation:

Let $\tilde{\mathbf{V}}^h_k(\pi_k) \in \mathbb{R}^{|\mathcal{S}|}$ be the value function vector when running $\pi_k$ in $\mathcal{M}_k$ from step $h$. It follows that:

\begin{equation*}
    \tilde{V}_k(\pi_k) = \langle P_0,  \tilde{\mathbf{V}}^0_k(\pi_k) \rangle
    \end{equation*}

Where $\tilde{\mathbf{V}}^h_k(\pi_k)[s]$ denotes the $s-$th entry of $\tilde{\mathbf{V}}^h_k(\pi_k)$ and $P_0$ is the initial state distribution. 

Let's also define a family of state action value function vector as $\tilde{\mathbf{Q}}_k^h(\pi_k), \hat{\mathbf{Q}}_k^h(\pi_k) \in \mathbb{R}^{| \mathcal{S} | \times |\mathcal{A}|}$ where $\tilde{\mathbf{Q}}_k^h(\pi_k)$ is the state value function of $\pi_k$ in $\mathcal{M}_k$ from step $h$. Similarly $\hat{\mathbf{Q}}_k^h(\pi_k)$ is the state value function of $\pi_k$ in $\hat{\mathcal{M}}_k$ from step $h$. 

We denote the $(s,a)-$th entry of $\tilde{\mathbf{Q}}_k^h(\pi_k)$ as $\tilde{\mathbf{Q}}_k^h(\pi_k)[s,a]$.

We condition on the following event $\mathcal{U}_k$ given by Corollary~\ref{corollary::dynamics_augmentation}:
\begin{small}
\begin{align*}
    \mathcal{U}_k :=& \{   \| \tilde{P}_k(s,a) - P(s,a) \|_1 \leq \left( 2\sqrt{|\mathcal{S}| \log\left(\frac{4| \mathcal{S}| |\mathcal{A}| K }{\delta}\right)} +1 \right)\beta_P\left(N_k(s,a),  \frac{ \delta}{2|S||A|}\right) \forall (s,a) \in \mathcal{S}\times \mathcal{A} \}
\end{align*}
\end{small}

We define the following convenient notation for the confidence intervals:
\begin{equation*}
    \Delta_{s,a,k} = \begin{cases}
                1 & \text{if } N_k(s,a) = 0\\
                \min(\left( 2\sqrt{|\mathcal{S}| \log\left(\frac{4| \mathcal{S}| |\mathcal{A}| K }{\delta}\right)} +1 \right)\beta_P\left(N_k(s,a),  \frac{ \delta}{2|S||A|}\right), 1) & \text{o.w.}
        \end{cases}
\end{equation*}

Notice that $\mathbb{P}(\mathcal{U}_k) \geq 1- \delta$. 

Notice $\Delta_{s,a,k} \leq \min( \frac{C }{N_k(s,a)}, 1)$ for $C =\tilde{\mathcal{O}}(|\mathcal{S}|) $. Where $\tilde{\mathcal{O}}$ hiddes logarithmic factors in $|\mathcal{S}|, |\mathcal{A}|,\delta$  and $T$ and independent of $t$.

We start by showing the following:

\begin{lemma}\label{lemma::bounding_q_h_t_recursively}
If $  \tilde{\mathbf{V}}_k^{h+1}(\pi_k) - \hat{\mathbf{V}}_k^{h+1}( \pi_k )= \delta^{h+1}_{t}(\pi_k) \in \mathbb{R}^{|\mathcal{S}|}$ then for all $(s,a) \in \mathcal{S} \times \mathcal{A}$:
\begin{equation*}
    \tilde{\mathbf{Q}}^{h}_k(\pi_k)[s,a] - \hat{\mathbf{Q}}_k^{h}(\pi_k)[s,a] \leq \mathbb{E}_{ s_{h+1} \sim P(s,a) } [\delta_k^{h+1}(\pi_k)[s_{h+1}] | (s,a) ] + \min(\frac{ C}{ \sqrt{ N_k(s,a) } }, 1)(H-h)
\end{equation*}
\end{lemma}

\begin{proof}
Notice that for all $(s,a) \in \mathcal{S} \times \mathcal{A}$:
\begin{align}
    \tilde{\mathbf{Q}}_k^{h}(\pi_k)[s,a] &= \tilde{r}(s, a ) + \langle \tilde{P}_k(s, a) , \tilde{\mathbf{V}}_k^{h+1}(\pi_k)\rangle \label{equation::vtilde_recursion}\\
    \hat{\mathbf{Q}}_k^{h}(\pi_k)[s,a] &= \tilde{r}(s, a ) + \langle P(s, a) , \hat{\mathbf{V}}_k^{h+1}(\pi_k)\rangle  \label{equation::vhat_recursion}
\end{align}
Therefore:

\begin{align*}
    \tilde{\mathbf{Q}}_k^{h}(\pi_k)[s,a] - \hat{\mathbf{Q}}_k^{h}(\pi_k)[s,a] &= \langle \tilde{P}_k(s, a) , \tilde{\mathbf{V}}_k^{h+1}(\pi_k)\rangle - \langle P(s, a) , \hat{\mathbf{V}}_k^{h+1}(\pi_k)\rangle \\
    &= \langle \tilde{P}_k(s,a) , \tilde{\mathbf{V}}_k^{h+1}(\pi_k)\rangle -      \langle P(s,a) , \tilde{\mathbf{V}}_k^{h+1}(\pi_k)\rangle  \\
    & \quad + \langle P(s, a) , \tilde{\mathbf{V}}_k^{h+1}(\pi_k)\rangle - \langle P(s, a) , \hat{\mathbf{V}}_k^{h+1}(\pi_k)\rangle\\
    &= \langle \tilde{P}_k(s, a) - P(s,a), \tilde{\mathbf{V}}_k^{h+1}(\pi_k) \rangle +  \langle P(s, a), \tilde{\mathbf{V} }_k^{h+1}(\pi_k) - \hat{\mathbf{V}}_k^{h+1}(\pi_k)\rangle \\
    &\leq \|\tilde{P}_k(s, a) - P(s,a) \|_1  \| \tilde{\mathbf{V}}_k^{h+1}(\pi_k) \|_\infty +  \mathbb{E}_{ s_{h+1} \sim P(s,a) } [\delta_k^{h+1}(\pi_k)[s_{h+1}] | (s,a) ]
\end{align*}

The last inequality follows by conditioning on the event the concentration bounds hold, since in this case $ \| \tilde{P}_k(s, a) - P(s, a) \|_1 \leq \min( \frac{C}{\sqrt{ N_k(a,s)}}, 1)$ and $\| \hat{\mathbf{V}}_k^{h+1}(\pi_k)\|_\infty \leq H-h$. The result follows. 
\end{proof}

Notice that $\tilde{\mathbf{V}}_k^{h}(\pi_k)[s] = \tilde{\mathbf{Q}}_k^h(\pi_k)[s, \pi_k(s)]$. This together with Lemma \ref{lemma::bounding_q_h_t_recursively} yields:

\begin{align}
   \delta_k^{h}(\pi_k)[s] &=   \tilde{\mathbf{V}}^{h}_k(\pi_k)[s] - \hat{\mathbf{V}}_k^{h}(\pi_k)[s] \notag\\
   &\leq   \mathbb{E}_{ s_{h+1} \sim P(s,\pi_k(s)) } [\delta_k^{h+1}(\pi_k)[s_{h+1}] | (s,\pi_k(s)) ]+ \min(\frac{ C}{ \sqrt{ N_k(s,\pi_k(s)) } }, 1)(H-h) \quad \forall s \in \mathcal{S}.\label{equation::bounding_value_function_recursive}
\end{align}

We further define:

\begin{align}
    \tilde{V}^h_k(\pi_k) &= \mathbb{E}_{\pi_k}[ \tilde{\mathbf{V}}_k^h(\pi_k)[s_h^{(k)}]  ] \label{equation::expected_value_h_tilde}\\
    \hat{V}^h_k(\pi_k) &= \mathbb{E}_{\pi_k}[ \hat{\mathbf{V}}_k^h(\pi_k)[s_h^{(k)}]  ]\label{equation::expected_value_h_hat}
\end{align}
Where the expectation is taken over the distribution of $s^{(k)}_h$ as encountered by policy $\pi_k$ when ran on $\mathcal{M}$. Combining the inequality in \ref{equation::bounding_value_function_recursive} and equations \ref{equation::expected_value_h_hat} and \ref{equation::expected_value_h_tilde} we obtain the following inequality:

\begin{equation*}
     \tilde{V}^{h}_k(\pi_k) - \hat{V}_k^{h}(\pi_k) \leq \tilde{V}^{h+1}_k(\pi_k) - \hat{V}_k^{h+1}(\pi_k) + \mathbb{E}_{\pi_k}\left[  \min(\frac{ C}{ \sqrt{ N_k(s_h^{(k)},\pi_k(s_h^{(k)})) } }, 1)(H-h) \right]. 
\end{equation*}

Applying this formula recursively from $h=H-1$ down to $h = 0$ yields:

\begin{align}
     A&= \tilde{V}_k(\pi_k) - \hat{V}_k(\pi_k) \notag\\
     &= 
     \tilde{V}^{0}_k(\pi_k) - \hat{V}^0_k(\pi_k) \notag\\
     &\leq \sum_{h=0}^{H-1} \mathbb{E}_{\pi_k}\left[ \min( \frac{ C}{ \sqrt{ N_k(s_h^{(k)},\pi_k(s_h^{(k)})) } }, 1)(H-h) \right] \notag\\
     &= \underbrace{ \mathbb{E}_{\pi_k}\left[  \sum_{h=0}^{H-1} \min(\frac{ C}{ \sqrt{ N_k(s_h^{(k)},\pi_k(s_h^{(k)})) } }, 1)(H-h) \right] }_{\spadesuit}   \label{equation::upper_bound_A_supporting}  . 
\end{align}

We proceed to bound term $\spadesuit$. Let:

\begin{equation*}
    Z_k = \mathbb{E}_{\pi_k}\left[  \sum_{h=0}^{H-1} \min( \frac{ C}{ \sqrt{ N_k(s_h^{(k)},\pi_k(s_h^{(k)})) } }, 1)(H-h) \right] -   \sum_{h=0}^{H-1} \min(\frac{ C}{ \sqrt{ N_k(s_h^{(k)},\pi_k(s_h^{(k)})) } }, 1)(H-h).
\end{equation*} 

Let $W_k = \sum_{\ell=1}^k Z_\ell$. In order to get rid of the uncommon terms that achieve high values in $Z_k$, we define $X_k$ as:

\begin{equation*}
    X_k = \mathbb{E}_{\pi_k}\left[  \sum_{h=0}^{H-1} \min( \Delta_{s_h^{(k)},\pi_k(s_h^{(k)}),k},  \frac{1}{H})(H-h) \right] -   \sum_{h=0}^{H-1} \min( \Delta_{s_h^{(k)},\pi_k(s_h^{(k)}),k},  \frac{1}{H})(H-h).
\end{equation*}

And define $Y_k = \sum_{\ell=1}^k X_\ell$ for all $k \geq 1$ with $Y_0 = 0$. Notice that:
\begin{equation}\label{equation::bounding_W_k_X_k}
    W_k - Y_k \leq \mathcal{O}( |\mathcal{S}|^2|\mathcal{A}| H  ) 
\end{equation}
This bound follows from the observation that whenever a pair $(s_h^{(k)}, \pi_k(s_h^{(k)})$ is encountered during the execution of policy $\pi_k$, its counter $N_k(s_h^{(k)}, \pi_k(s_h^{(k)})$ is incremented, thus inside the expectation $\mathbb{E}_{\pi_1, \cdots, \pi_K}$ the occurrence of $N_k(s_h^{(k)}, \pi_k(s_h^{(k)}))= j$ can be thought of as happening only for one $k$. Therefore, at most for each state action pair we have a difference of $\sum_{\ell  = 1}^{|\mathcal{S}|^2H^2 }\frac{1}{\ell} \approx |\mathcal{S}|H $.

It is easy to see that $Y_k$ is a martingale satisfying the following bounded differences condition:
\begin{equation*}
    | X_k | \leq 2H 
\end{equation*}
A simple use of Azuma-Hoeffding yields that for any $\delta \in (0,1)$ and with probability at least $1-\delta$:

\begin{equation*}
    Y_K \leq  \frac{H}{2}\sqrt{2K \log\left(\frac{1}{\delta} \right) }
\end{equation*}
Consequently:
\begin{align*}
    W_K = \sum_{k=1}^K \left(   \mathbb{E}_{\pi_k}\left[  \sum_{h=0}^{H-1} \min(\frac{ C}{ \sqrt{ N_k(s_h^{(k)},\pi_k(s_h^{(k)})) } }, 1)(H-h) \right] -   \sum_{h=0}^{H-1} \min( \frac{ C}{ \sqrt{ N_k(s_h^{(k)},\pi_k(s_h^{(k)})) } }, 1)(H-h) \right) \\
    \leq \frac{H}{2} \sqrt{2K \log\left(\frac{1}{\delta} \right) } + \mathcal{O}( |\mathcal{S}|^2|\mathcal{A}| H ).
\end{align*}

And therefore with probability at least $1-\delta$:

\begin{align}
     \sum_{k=1}^T    \mathbb{E}_{\pi_k}\left[  \sum_{h=0}^{H-1} \min(\frac{ C}{ \sqrt{ N_k(s_h^{(k)},\pi_k(s_h^{(k)})) } }, 1)(H-h) \right] &\leq \sum_{k=1}^K \sum_{h=0}^{H-1} \min(\frac{ C}{ \sqrt{ N_k(s_h^{(k)},\pi_k(s_h^{(k)})) } }, 1)(H-h) \notag \\
     &\quad+ \frac{H}{2} \sqrt{2K \log\left(\frac{1}{\delta} \right) } + \mathcal{O}( |\mathcal{S}|^2|\mathcal{A}| H ) \notag\\
     & \leq H \sum_{k=1}^K \sum_{h=0}^{H-1} \min(\frac{ C}{ \sqrt{ N_k(s_h^{(k)},\pi_k(s_h^{(k)})) } }, 1)  \notag  \\
     &\quad+ \frac{H}{2} \sqrt{2K \log\left(\frac{1}{\delta} \right) } + \mathcal{O}( |\mathcal{S}|^2|\mathcal{A}| H )\notag\\
     &\leq  CH \sum_{s,a \in \mathcal{S}\times\mathcal{A}}  \sum_{\ell=1}^{N_K(s,a)} \frac{ 2}{ \sqrt{ \ell } } + \frac{H}{2} \sqrt{2K \log\left(\frac{1}{\delta} \right) } \notag\\
     &+ \mathcal{O}( |\mathcal{S}|^2|\mathcal{A}| H )\notag\\
     &\leq CH \sqrt{ |\mathcal{S}| | \mathcal{A}| KH   }  +  \frac{H}{2} \sqrt{2K \log\left(\frac{1}{\delta} \right) } \notag \\
     &+\mathcal{O}( |\mathcal{S}|^2|\mathcal{A}| H ).\label{equation::inequality_A_support}
\end{align}

Combining inequalities \ref{equation::upper_bound_A_supporting} and \ref{equation::inequality_A_support} with all these results yields the following bound for term $A$:

\begin{align*}
A&=  \sum_{k=1}^K \tilde{V}_k(\pi_k)  - \hat{V}_k(\pi_k) \\
&\leq CH \sqrt{ |\mathcal{S}| | \mathcal{A}| KH   }  +  \frac{H}{2} \sqrt{2K \log\left(\frac{1}{\delta} \right) } + \mathcal{O}( |\mathcal{S}|^2|\mathcal{A}| H )\\
&= \tilde{\mathcal{O}}( |\mathcal{S}|H\sqrt{|\mathcal{S}||\mathcal{A}|H K } + H \sqrt{K} +  |\mathcal{S}|^2|\mathcal{A}| H)
\end{align*}

\subsection{Estimation Error Main}\label{subsection::estimation}

The general idea behind bounding the error in term $\mathrm{II}$ of Equation \ref{equation::regret_decomposition}, is similar to what UCRL does. See the details in Appendix~\ref{section::etimation_error_appendix}. We show that with probability at least $1-2\delta K$ term $\mathrm{II}$ admits the following bound:
\begin{equation*}
    \mathrm{II} \leq \tilde{\mathcal{O}}( |\mathcal{S}|H\sqrt{|\mathcal{S}||\mathcal{A}|T }   ) 
\end{equation*}
where $\tilde{O}$ hides logarithmic factors in $|\mathcal{A}|, |\mathcal{S}|, \delta$ and $K$. This concludes the proof of Theorem \ref{theorem::main_gaussian}.

\subsection{UCRL NARL Main Theorem}

Putting our bounds together for term $A$ and $B$ yields:

\begin{align*}
    R(T) &\leq \tilde{\mathcal{O}}( |\mathcal{S}|H\sqrt{|\mathcal{S}||\mathcal{A}|H K } +  |\mathcal{S}|^2|\mathcal{A}| H  )  \\
    &= \tilde{\mathcal{O}}( |\mathcal{S}|H\sqrt{|\mathcal{S}||\mathcal{A}|T } +  |\mathcal{S}|^2|\mathcal{A}| H  )
\end{align*}

\subsection{UCBVI NARL Main Theorem}

Recall that  Noise Augmented Value Iteration (NAVI) proceeds as follows: at the beginning of episode $k$ we compute a $Q-$function $\tilde{\mathbf{Q}}_k$ as:
\begin{align*}
    \tilde{\mathbf{Q}}_{k,h}(s,a) &=\min\Big( \tilde{\mathbf{Q}}_{k-1, h}(s,a), H, \tilde{r}_k(s,a) +  \quad\mathbb{E}_{s' \sim \hat{P}_k(s,a)} \left[  \tilde{V}_{k, h+1}(s,a)  \right]   \Big) \notag\\ 
    \tilde{\mathbf{V}}_{k,h}(s,a) &= \max_{a \in \mathcal{A}}   \tilde{\mathbf{Q}}_{k,h}(s,a) 
\end{align*}

We make use of the following theorem for Gaussian reward augmentation:

\begin{lemma}\label{lemma::anticoncentration_rewards_gaussian_UCBVI_appendix}
 Let $(s,a) \in \mathcal{S} \times \mathcal{A}$. If $\tilde{r}^{(m)}_k(s,a) \sim \hat{r}_k(s,a) + \mathcal{N}(0, \sigma^2)$ for $\sigma = 2H \beta_r(N_k(s,a),  \frac{ \delta}{2|S||A|}) $ then:
 \vspace{-2mm}
\begin{equation}
    \mathbb{P}(\tilde{r}^{(m)}_k(s,a ) \geq r(s,a) + b_k(s,a)  | \mathcal{E}) \geq \frac{1}{10}.
\end{equation}
Where $b_k(s,a) = 7HL\sqrt{\frac{1}{N_k(s,a)}}$ where $L = \ln(5SAKH/\delta)$.
\end{lemma}

The same proof as in Lemma~\ref{lemma::anticoncentration_rewards_gaussian_appendix} yields the result. 

As a consequence of this result it is easy to see that making use of multiple samples provides concentration and optimisim with enough probability. Following the logic in the proof of UCBVI guarantee with bonus 1 in \cite{azar2017minimax} yields the proof of Theorem~\ref{theorem::main_gaussian_UCBVI}.

\section{Additional Experimental Results}

\paragraph{Delusional World Models}

In Fig. \ref{figure:delusional} we show the learning curves from the nine different configurations tested on the InvertedPendulum task, shown in Table \ref{table:navi_results}. In green we show the mean reward in the environment, and we see that some settings fail to learn. In red we show the reward for the policy inside the model. This clearly shows that the models are being exploited, as in some cases the model shows a high reward yet the policy is getting close to zero in the true environment (e.g. $K=10, \epsilon_K=\mathrm{None}$).

\begin{figure}[H]
    \begin{minipage}{0.99\textwidth}
    \centering\subfigure{\includegraphics[width=.9\linewidth]{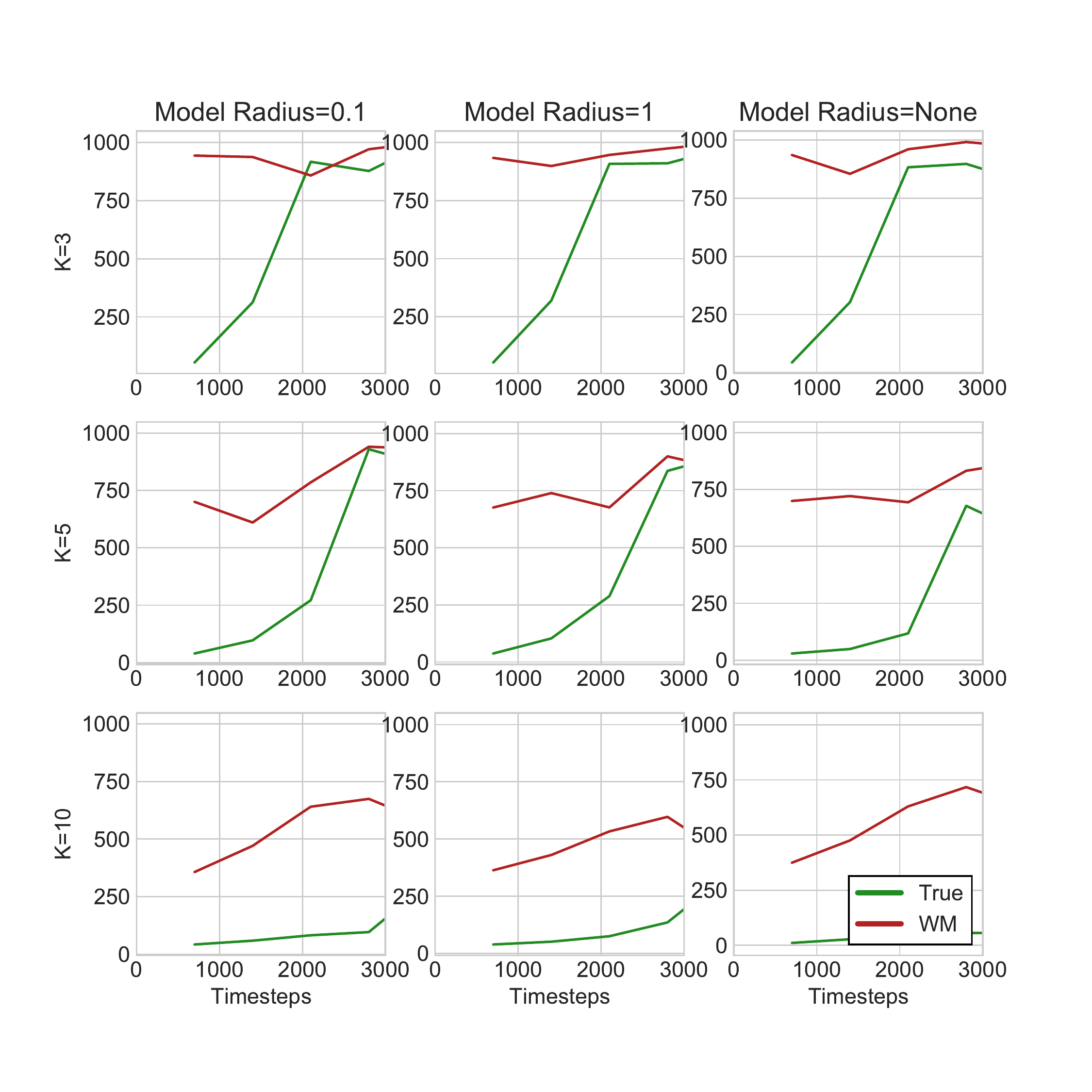}}
    \caption{\small{Mean policy performance inside the world model (WM, green), and in the true environment (red), for $3000$ timesteps in the $\mathrm{InvertedPendulum}$ task.}}
    \label{figure:delusional}
    \end{minipage}
\end{figure}

\section{Implementation Details}

We adapt our code from our own PyTorch implementation of MBPO \cite{whentotrust}. Most hyperparameters were chosen as in their paper. The main new hyperparameter is $\epsilon_M$, where we include all our results in Table \ref{table:navi_results}. When optimizing the policy inside the model, we use the procedure in \ref{appendix:espi}.

\section{Extended Soft Policy Iteration}\label{appendix:espi}

For the Deep RL experiments, we need to consider that in MBPO-based approaches, a policy is learned using soft policy iteration, which aims to maximise not pure cumulative return, but cumulative return and some total entropy per timestep:
\begin{align*}
    \pi^* = \arg \max \sum_{t} \mathbb{E}_{(\mathbf{s}_t,\mathbf{a}_t)\sim \rho_\pi} \left[ r(\mathbf{s}_t, \mathbf{a}_t) + \alpha \mathcal{H}(\pi(\cdot|\mathbf{s}_t)) \right].
\end{align*}
We must therefore integrate the policy iteration used in UCRL2 with this new entropy-dependency in order to accurately assess optimism over value and entropy.

We first write out soft-policy iteration:
\begin{align*}
    \mathcal{T}^\pi Q(s,a) =& r(s,a) + \gamma \mathbb{E}_{s'\sim p} [ V(s') ]\\
    V(s) =& \mathbb{E}_{a \sim \pi} [ Q(s,a) - \alpha \pi(a|s)]
\end{align*}
where $\mathcal{T}^\pi$ is the modified Bellman operator.

We then note it is possible to write extended policy iteration from UCRL2 using Bellman operator notation:
\begin{align*}
    \mathcal{T}^\pi V(s) = \max_{a \in A} \left\{ \Tilde{r}(s,a) + \max_{p \in P(s,a)} \left\{ \sum_{s' \in S} p(s')V(s') \right\} \right\}
\end{align*}
Combining these two approaches, one possibility is as follows, which we will term Extended Soft Policy Iteration (ESPI):
\begin{align*}
    \mathcal{T}^\pi Q(s,a) =& \Tilde{r}(s,a) + \gamma \max_{p \in P(s,a)} \left\{ \mathbb{E}_{s'\sim p} [ V(s') ] \right\} \\
    V(s) =& \max_{a \in A} \left\{ Q(s,a) - \alpha \log \pi(a|s)] \right\}\label{eqn:ESVI}
\end{align*}
However this is intractable in a continuous control; we relax this to the tractable SPI approach:
\begin{align*}
    V(s) =& \mathbb{E}_{a \sim \pi} [ Q(s,a) - \alpha \log \pi(a|s)].
\end{align*}

Observing the Bellman operator term in ESPI, we note that there are two key differences to standard SPI: 1) we take some optimistic reward per state-action pair; 2) we take the most optimistic dynamics model which maximises the expected return from the next state.

In order to implement the former, we apply the max over the reward predictions from the models subject to the model radius parameter $\epsilon_M$.

In order to implement the max over models, we generate next states across all models, pass those next states through the actor to generate the stochastic action, then calculate the expected `soft' return across off models by passing the resultant next states and actions through the critic, subject to each action's respective entropy. Since action selection is stochastic, it is possible to generate multiple samples over the policy next actions to acquire a more accurate estimation of the soft expected return, but we found sampling the actor once was sufficient in practice.

\end{document}